\documentclass[preprint,nopreprintline,5p,twocolumn]{elsarticle}

\usepackage{amssymb}
\usepackage{amsmath}
\usepackage{subfig}
\usepackage{soul}
\soulregister\ref7
\soulregister\cite7
\soulregister\citealt7

\newcommand{\mioldIII}[1]{\iffalse{#1}\fi}
\usepackage{amsthm}
\usepackage{tabularray}

\usepackage{multirow}
\usepackage[table,xcdraw]{xcolor}
\usepackage[normalem]{ulem}
\usepackage{graphicx}
\useunder{\uline}{\ul}{}
\usepackage{booktabs}
\usepackage{makecell}

\newtheorem{theorem}{Theorem}[section] 
\newtheorem{definition}[theorem]{Definition} 
\newtheorem{lemma}{Lemma}

\begin{document}

\begin{frontmatter}



\title{A Hierarchical Framework for Measuring Scientific Paper Innovation via Large Language Models}

\author[label1,label2]{Hongming Tan\fnref{co-first}}
\ead{thm22@mails.tsinghua.edu.cn}

\author[label1]{Shaoxiong Zhan\fnref{cofirst}}
\ead{zhansx24@mails.tsinghua.edu.cn}

\author[label1]{Fengwei Jia}
\ead{okweb@163.com}

\author[label1,label2]{Hai-Tao Zheng\corref{cor}}
\ead{zheng.haitao@sz.tsinghua.edu.cn}

\author[label1,label2]{Wai Kin (Victor) Chan\corref{cor}}
\ead{chanw@sz.tsinghua.edu.cn}

\fntext[co-first]{Equal Contribution.}
\cortext[cor]{Corresponding author: Wai Kin (Victor) Chan and Hai-Tao Zheng.}

\address[label1]{Shenzhen International Graduate School, Tsinghua University, Shenzhen 518055, China}
\address[label2]{Pengcheng Laboratory, Shenzhen, 518055, China}

% \author[label1,label2]{Hongming Tan}
% \affiliation[label1]{organization={},

% \affiliation[label2]{organization={},

% \author[label1]{Shaoxiong Zhan} %% Author name

% \affiliation{organization={},%Department and Organization

%% Abstract
\begin{abstract}
Measuring scientific paper innovation is both important and challenging. Existing content-based methods often overlook the full-paper context, fail to capture the full scope of innovation, and lack generalization. We propose HSPIM, a hierarchical and training-free framework based on large language models (LLMs). It introduces a Paper-to-Sections-to-QAs decomposition to assess innovation. 
We segment the text by section titles and use zero-shot LLM prompting to implement section classification, question-answering (QA) augmentation, and weighted innovation scoring.
The generated QA pair focuses on section-level innovation and serves as additional context to improve the LLM scoring. For each chunk, the LLM outputs a novelty score and a confidence score. We use confidence scores as weights to aggregate novelty scores into a paper-level innovation score. To further improve performance, we propose a two-layer question structure consisting of common and section-specific questions, and apply a genetic algorithm to optimize the question-prompt combinations. 
Furthermore, under the fine-grained structure of innovation, we extend HSPIM to an HSPIM$^+$ that generates novelty, contribution, and feasibility scores with respective confidence scores.
Comprehensive experiments on scientific conference paper datasets show that HSPIM outperforms baseline methods in effectiveness, generalization, and interpretability.
Demo code is available at https://github.com/Jasaxion/HSPIM.
\end{abstract}

%Graphical abstract
% \includegraphics[width=1\textwidth]{graph_abstract.pdf}
% % \includegraphics{grabs}


% \begin{highlights}
% \item HSPIM measures paper innovation via a Paper-to-Sections-to-QAs hierarchy.
% \item HSPIM is the first LLM-based framework for scientific innovation measurement.
% \item Innovation is distinguished from mere novelty in the framework design.
% \item Section-based weighted scoring with QA augmentation extracts innovation signals.
% \item A genetic algorithm jointly optimizes multi-prompt questions.
% \end{highlights}

%% Keywords
\begin{keyword}
Innovation Measurement \sep Hierarchical Framework \sep Text Mining \sep Large Language Model
\end{keyword}

\end{frontmatter}

% \linenumbers


\section{Introduction}
% \label{sec1}
Recently, the rapid increase in scientific paper submissions has significantly heightened the demand for high-quality evaluation.
The most important aspect of evaluating papers is to assess their innovation.
Innovation was introduced in economics to explain how new products, processes, or structures emerge to transform industries and promote development \cite{schumpeter2021theory}.
In terms of research, innovation involves applying novel tools, methods, or concepts to enhance both methodologies and outcomes \cite{costanzo2019innovation}. 
We note that \textbf{innovation} is different from novelty. While novelty focuses on newness and uniqueness, it may lack practical feasibility and real-world contributions.
Instead, innovation emphasizes both novelty and practical value \cite{scott1994determinants}.

Current methods mainly analyze how research combines existing knowledge in unconventional ways to assess novelty rather than innovation.
These approaches can be categorized into content-based and metadata-based methods. Content-based methods analyze textual features to assess novelty directly from the text, utilizing rule-based systems \cite{basu2001evaluating} and deep learning models \cite{amplayo2018network}. In contrast, metadata-based approaches rely on external attributes such as journal impact and author reputation \cite{boyack2005predicting, tahamtan2018creativity}.

Content-based methods help to evaluate unpublished papers and to support double-blind review. They are therefore more practical than metadata-based methods. However, current content-based methods still have several limitations.
First, to handle long text, most methods focus on specific sections such as titles, abstracts, keywords, or contribution sentences, which may cause biases compared to full paper analysis \cite{pitkin1999accuracy}. 
Second, many approaches evaluate only the novelty of a paper without considering its value and feasibility, leading to an incomplete assessment of innovation.
Third, existing models are typically trained on papers from specific topics and publication years, limiting their ability to build a generalizable paper innovation measurement framework.

To address the aforementioned issues in content-based methods, we propose a hierarchical framework for measuring scientific paper innovation based on large language models (LLMs). 
This framework evaluates the innovation of scientific papers at multiple levels: from the document level (full paper) to the paragraph level (sections) and finally to the sentence level (question-answer pairs).
First, to handle the problem of long-text analysis, we divide papers into text chunks based on their original smallest sections, which are then categorized by LLMs into common scientific paper section types following the IMRaD format \cite{sollaci2004introduction}.
Specifically, we propose to divide the scientific paper into the following sections: Abstract, Introduction, Related Work, Approach/Methodology/Model/Method, Analysis Theory, Experiments, Experiment Analysis, Discussion/Limitations, and Conclusion.
Initially, we name this approach \textbf{Section-based} Scientific Paper Innovation Measurement (SSPIM).

Second, to measure innovation beyond novelty, we follow social science theories to define it as a combination of \textit{novelty} and \textit{practical value}. Practical value includes \textit{contribution} and \textit{feasibility}. This structure reflects common understanding in innovation studies. Based on this innovation concept, we propose a \textbf{weighted innovation scoring} formula with \textbf{question-answering (QA) augmentation}. (i) We input each text chunk into a zero-shot LLM prompt to generate a JSON output \{\textit{``novelty\_score''}, \textit{``reason''}, \textit{``confidence\_score''}\}. 
The \textit{novelty\_score} reflects section-level novelty. The \textit{reason} explains the rationale behind the score. The \textit{confidence\_score} serves as a weight to modulate the novelty score of each section. Specifically, the confidence score reflects how strongly the LLM supports the local novelty claim in context. This design implicitly captures the feasibility attribute of innovation by indicating the plausibility of the novelty score.
We then apply weighted averaging using the confidence scores to aggregate local novelty into a paper-level innovation score.

(ii) We perform QA augmentation before scoring. 
Specifically, we design questions to guide the LLM to generate innovation-focused answers for each section chunk. The resulting QA pair is added as additional context to enhance in-context learning during LLM scoring.
Inspired by structural and regulatory genes in genetics\mioldIII{ \cite{spitz2012transcription}}, we design a \textit{two-layer question structure}: one common question shared across all sections, and one specific question for each section type.
The generated QA pairs distill original section content into concise, innovation-focused carriers \cite{tan2024qaea}, helping the LLM analyze more attributes of innovation, such as feasibility and contribution. 
We refer to this Paper-to-Sections-to-QAs framework as \textbf{Hierarchical} Scientific Paper Innovation Measurement (HSPIM).
To evaluate the effect of a fine-grained representation of innovation, we further extend HSPIM to HSPIM$^+$ by incorporating a $p$-norm aggregation of novelty, contribution, and feasibility scores, each with corresponding confidence weights.

Third, we use pre-trained LLMs as zero-shot generators and scorers to exploit their knowledge and reasoning strength.
This approach is training-free and demonstrates strong generalization. To further improve robustness and effectiveness, we treat the question prompts as optimization targets in a combinatorial optimization task. Inspired by previous prompt optimization approaches \cite{guoconnecting, yang2023large} and gene variation theory, we propose the \textbf{multi-prompt optimizer} using the heuristic genetic algorithm (GA) \cite{mitchell1998introduction}. We assume there are \(s\) section types. 
GA searches for the optimal question-prompt combination within \(s\) specific question sets and one common question set, following standard genetic algorithm steps: population initialization, selection, crossover, and mutation.
Each individual in the population consists of \(s\) specific questions plus one common question. Under the elite GA, we explore different optimization strategies for two-layer question structure, including joint optimization, two-step optimization, and pruning, to further enhance performance.

We use three peer review datasets from PeerRead \cite{kang2018dataset} and one dataset from NLPeer \cite{dycke2023nlpeer} to conduct extensive experiments.
We average peer review scores for originality and soundness/correctness as proxies for \textit{novelty} and \textit{practical value}, respectively, to construct the ground-truth innovation score.
Performance is evaluated using metrics such as RMSE and MAE, around which we conduct further detailed analyses.
Additionally, we evaluate the semantic textual similarity between \textit{reason} output scoring reasons and actual peer review comments.

In addition to traditional content-based approaches, recent LLM-based models such as SEA-E \cite{yu2024automated} focus on generating peer reviews with attributes like originality and soundness. However, these systems are typically fine-tuned on specific datasets, operate in a single-pass manner, and capture review-style criteria rather than directly measuring innovation. 
In contrast, our framework represents the first LLM-based approach for scientific innovation measurement, realized through a training-free, hierarchical evaluation that integrates section-level QAs with a genetic multi-prompt optimizer.

In summary, the main contributions of this paper are as follows:  
\begin{itemize}
    \item We propose a Hierarchical Scientific Paper Innovation Measurement (HSPIM) framework based on Paper-to-Sections-to-QAs decomposition.
    To the best of our knowledge, HSPIM is the first section-based text mining framework that employs zero-shot LLMs to evaluate scientific paper innovation.
    It effectively addresses challenges in long-text analysis, innovation quantification, and model generalization.
    \item We clearly define the concepts of innovation by distinguishing it from novelty. Based on this distinction, we propose weighted innovation scoring and QA augmentation.
    \item We introduce a genetic algorithm-based prompt optimizer for our two-layer question structure to optimize LLM scoring.
    We are the first to apply genetic algorithms for joint multi-prompt optimization.
    \item 
    We conduct both theoretical analysis and empirical validation to demonstrate the effectiveness and robustness of HSPIM.
\end{itemize}

\section{Related Work}
Our work draws on prior research on three areas: the concept of innovation, paper innovation measurement, and prompt-based LLMs.
\subsection{Innovation Concept}
Schumpeter initially defined innovation in \textit{The Theory of Economic Development} as creating new products, processes, or structures to disrupt equilibrium and drive growth \cite{schumpeter2021theory}. 
Christensen further distinguished innovation into sustaining and disruptive types, emphasizing that disruptive innovation combines novelty with real-world impact through market reshaping and widespread adoption \cite{bower1995disruptive}.
These classic economic theories highlight that innovation inherently focuses on creating value through real-world application and impact.

Social science research provides multifaceted insights into the nature of innovation. Early work highlights that innovation should respond to real-world needs through public demand and institutional frameworks\mioldIII{ \cite{edler2007public}}. 
Later studies suggest that innovation goes beyond novelty, encompassing practical value within broader ecological and methodological contexts \cite{alajami2020beyond}.
Other research shows that innovation is socially constructed and requires recognition from audiences and institutions\mioldIII{ \cite{godart2020sociology}}. Recent work in research assessment suggests that innovation involves both methodological novelty and impact-oriented contributions \cite{costanzo2019innovation}. Together, these studies support the view that innovation entails both novelty and practical value.

Historically, innovation evaluation focused mainly on novelty. Over time, it began to include practical value, such as applicability and feasibility, recognizing that innovation is a multifaceted concept\mioldIII{ \cite{godin2008innovation}}. 
Dosi introduced the concepts of \textit{technological paradigms} and \textit{technological trajectories}, clarifying that innovation involves not just novelty, but also systematic realizability and practical feasibility \cite{dosi1982technological}. 
Similarly, Rogers et al.\ argued that innovation diffusion depends on two key factors. The first is \textit{relative advantage}, which reflects real-world contribution. The second is \textit{complexity and trialability}, which reflect feasibility \cite{rogers2014diffusion}.

\begin{figure*}[ht]
  \centering
  \includegraphics[width=0.5\textwidth]{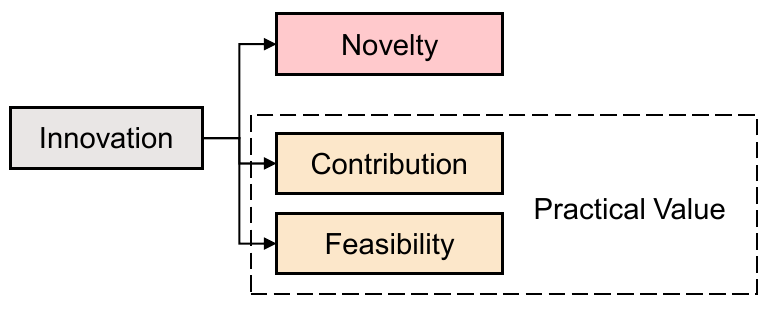}
  \caption{Conceptual decomposition of innovation into novelty and practical value (contribution and feasibility), which grounds the scoring design of HSPIM.}
  \label{innovation_concept}
\end{figure*}

Integrating these insights, we define innovation as a composite concept with two main components: \textit{Novelty} and \textit{Practical Value}. Practical value is further divided into \textit{Contribution} (theoretical, methodological, or applied value) and \textit{Feasibility} (the potential for real-world implementation), as shown in Fig. \ref{innovation_concept}. Since innovation is a complex concept, it cannot be directly measured like novelty. Therefore, a comprehensive analysis is needed to capture its full structure before quantification.

\subsection{Approaches to Measuring Paper Innovation}
Existing methods for measuring paper innovation mainly focus on novelty. As previous studies suggest, paper novelty is evaluated by identifying atypical combinations of existing knowledge \cite{mayer1995search, simonton2004creativity}. Approaches to novelty measurement include content-based and metadata-based methods.

Content-based methods rely on textual analysis to align with peer review practices that assess novelty directly from the text. 
These content-based models range from rule-based approaches to deep learning methods.
Rule-based methods include WordNet \cite{basu2001evaluating} and domain-specific knowledge bases \cite{shabani2023rule}.
Deep learning models include unsupervised generative models such as graph-based autoencoder neural networks \cite{amplayo2018network} and supervised discriminative models like convolutional neural network (CNN) \cite{lecun1989backpropagation}, long short-term memory (LSTM) \cite{hochreiter1997long, kang2018dataset}, Transformer \cite{vaswani2017attention},
TextCNN and BERTopic \cite{jeon2023measuring, wang2023measuring}.
Besides, recent studies have been conducted on text novelty detection tasks~\cite{shah2021three}.
Recently, fine-tuned LLM-based models like SEA \cite{yu2024automated} generate review texts along with several evaluation scores to simulate the real-world peer review.

Content-based methods can be categorized by their analysis target, such as topics \cite{wang2024effective}, titles \cite{jeon2023measuring}, keywords \cite{boudreau2016looking}, paragraphs\mioldIII{ \cite{hou2022new}}, or contribution sentences \cite{wang2023measuring}. Focusing on specific parts of the paper may introduce biases \cite{pitkin1999accuracy} compared to using full-text analysis.

On the other hand, metadata-based methods evaluate external factors such as journal impact and author reputation through regression models \cite{boyack2005predicting}. Other metadata-based methods includes applying scientometrics techniques \cite{tahamtan2018creativity} and bibliometric methods \cite{matsumoto2021introducing} based on reference analysis.
Since content-based methods help to evaluate unpublished papers and support double-blind peer review, they are more practical than metadata-based methods for real-world applications.
In addition, the innovation level of a paper does not depend on its external attributes.

We propose a content-based HSPIM framework using training-free LLM prompting to extract deep and comprehensive innovation information through section-based QA. 
Our approach addresses the challenges of long-text analysis, innovation quantification, and model generalization in content-based methods.

\subsection{Prompt-based LLMs}
% \textbf{Prompt-based LLMs:} 
Large language models (LLMs) are effective tools for in-context learning \cite{kojima2022large}. Zero-shot LLMs combine input data and prompt design to form context and then use this context to generate answers directly.
Moreover, LLM-as-a-Judge \cite{zheng2023judging} shows that zero-shot LLMs can closely match human judgment and produce reliable scores. 
LLM-as-a-Judge can directly assign quantitative scores based on in-context learning.
Thus, we suggest that zero-shot LLMs can effectively generate novelty and confidence scores for evaluating scientific paper innovation based on the paper content.
Furthermore, we extract innovation-focused QA pairs from the original text chunk of the paper for better innovation judgment \cite{tan2024qaea}.

As previous studies have shown \cite{liu2023pre}, LLM performance is sensitive to prompt designs.
Therefore, prompt optimization has emerged as a prominent paradigm.
Early works such as Optimization by PROmpting (OPRO) \cite{yang2023large} and Automatic Prompt Optimization (APO)\mioldIII{ \cite{pryzant2023automatic}} leverage zero-shot LLMs as optimizers to update prompts based on output performance feedback. They apply heuristic algorithms to automate this process, enabling derivative-free and training-free optimization that surpasses manual designs.
Other related methods include Evolutionary Prompt Optimizer (EvoPrompt) \cite{guoconnecting} and AutoHint \cite{sun2023autohint}. 
\mioldIII{Another line of research explores prompt-tuning with domain knowledge. For example, BioKnowPrompt uses domain knowledge to improve relation extraction \cite{li2022bioknowprompt}.}

In contrast to long-prompt optimization\mioldIII{ \cite{hsieh2023automatic}}, which updates individual sentences within a large prompt in each iteration while maintaining the overall structure, our approach performs multi-prompt optimization. Instead of modifying a single long prompt, we need to jointly optimize multiple question prompts, each generating QA pairs for different text chunks. These QA pairs contribute to weighted innovation scoring to compute an innovation score, and we update question prompts based on RMSE performance feedback. To achieve this, we introduce a genetic algorithm-based method \cite{mitchell1998introduction} that simultaneously updates multiple prompts within HSPIM.

\section{Methodology}
In this section, we first define the notations and the problem. 
Then, we provide a detailed introduction to the HSPIM framework and compare it with its simplified variant. Finally, we justify the two-layer question structure and propose three multi-prompt optimization strategies.
A summary of key notations used in this paper is given in Table \ref{tab:notation-summary}.

\subsection{Notations and Problem Definition}
\label{sec:notations}
In this paper, we focus on measuring scientific paper innovation to quantify this challenging attribute. A scientific paper dataset consists of a collection of paper texts \(\mathcal{C}\) and the corresponding peer reviews. Let \(\mathcal{C} = \{p_1, p_2, \ldots, p_n\}\) denote a set of papers, where each \(p_i\) is a paper and \(n\) is the total number of papers. 
Our objective is to assign numeric innovation scores to all papers in \(\mathcal{C}\) and to minimize the prediction error between the predicted scores \(\hat{\textit{Innovation}}_i\) and the ground-truth scores \(\textit{Innovation}_i\).

We derive ground-truth innovation scores from peer reviews in the scientific paper dataset. Each paper \(p_i\) receives \(m_i\) reviews \(\{r_{i1}, r_{i2}, \dots, r_{im_i}\}\), each containing a comment and several evaluation scores. Following established peer review criteria, we quantify innovation using originality and soundness/correctness. 
Specifically, for each review \(r_{ij}\), the innovation score is defined as \(\textit{Innovation}_{ij} = (\textit{Originality}_{ij} + \textit{Soundness}_{ij})/2\).
The final labeled innovation score for \(p_i\) is obtained by averaging its review scores: \(\textit{Innovation}_i = \frac{1}{m_i} \sum_{j=1}^{m_i} \textit{Innovation}_{ij}\). Our objective is to minimize the discrepancy between the predicted innovation score \(\hat{\textit{Innovation}}_i\) and the labeled innovation score \(\textit{Innovation}_i\). 
To achieve this, we use the root mean squared error (RMSE), calculated as the square root of the average of \((\hat{\textit{Innovation}}_i - \textit{Innovation}_i)^2\) over all \(n\) papers.

\begin{table}[ht]
\centering
\caption{Summary of notations in HSPIM framework.}
\label{tab:notation-summary}
\scriptsize
\setlength{\tabcolsep}{0.4pt}
% \resizebox{\textwidth}{!}{%
\begin{tabular}{ll}
% \begin{tabular}{lp{1cm}}
\toprule
\textbf{Notation}        & \textbf{Meaning} \\ 
\midrule
$\mathcal{C}$            & Collection of all papers \\
$p_i$                    & The $i$-th paper in $\mathcal{C}$ \\
$n$                      & Total number of papers, i.e.\ $|\mathcal{C}|$ \\
$m_i$                    & Number of peer reviews for paper $p_i$ \\
$s$                      & Number of predefined section types \\
$l_i$                    & Number of text chunks in paper $p_i$ \\
$t_{ik}$                 & The $k$-th text chunk of paper $p_i$ \\
$\textit{Originality}_{ij}$, $\textit{Soundness}_{ij}$ & Peer-review scores for paper $p_i$ from review $j$ \\
% $\textit{Innovation}_{ij}$ & Labeled innovation score for review $j$ of paper $p_i$ \\
$\textit{Innovation}_{i}$  & Labeled innovation score for paper $p_i$ \\
$\hat{\textit{Innovation}}_{i}$ & Predicted innovation score for paper $p_i$ \\
$\textit{Confidence}_{ik}$ ($c_{ik}$) & Predicted confidence score for chunk $t_{ik}$ \\
$\textit{Novelty}_{ik}$ ($n_{ik}$)  & Predicted novelty score for chunk $t_{ik}$ \\
$\textit{Contribution}_{ik}$ ($d_{ik}$) & Predicted contribution score for chunk $t_{ik}$ \\
$\textit{Feasibility}_{ik}$ ($f_{ik}$)  & Predicted feasibility score for chunk $t_{ik}$ \\
$Q_r$                    & Specific question set for section type $r$ \\
$\mathcal{Q}_c$          & Common question set \\
\bottomrule
\end{tabular}
% }
\end{table}

\begin{figure*}[ht]
  \centering
  \includegraphics[width=1\textwidth]{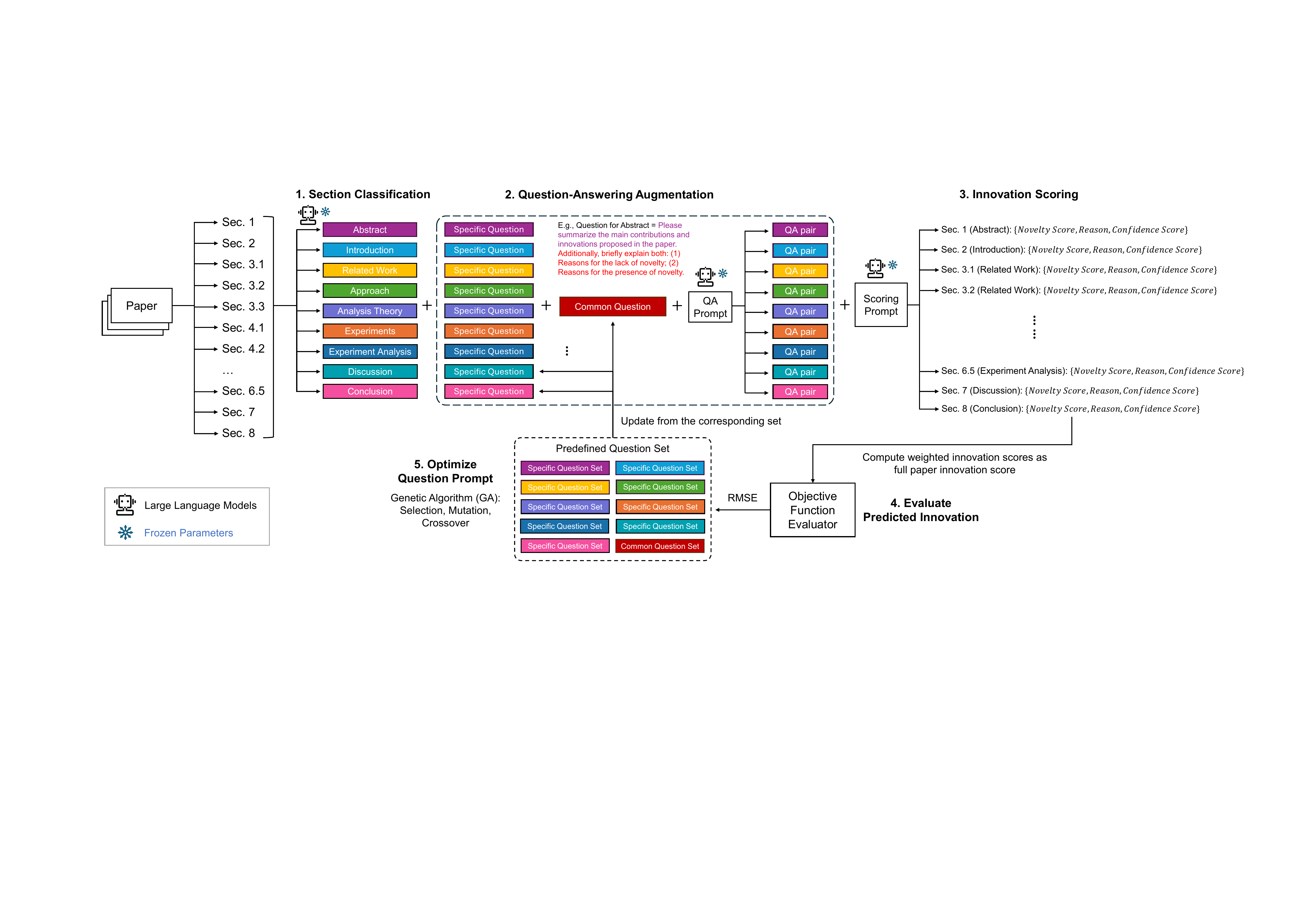}
  \caption{An example of hierarchical scientific paper innovation measurement (HSPIM) via large language models. 
  We use zero-shot LLMs in three steps: section classification (Step-1), question answering (Step-2), and innovation scoring (Step-3). 
  The overall paper innovation score is computed based on the novelty and confidence scores. In Step-4, we calculate the RMSE between this score and the ground truth.
  To improve in-context learning, we design two types of innovation-related questions: a common question applied to all sections and specific questions for each section type. In Step-5, we use a Genetic Algorithm (GA) to find better question prompts and update them simultaneously.
  }
  \label{HSPIM_framework}
\end{figure*}

We propose the HSPIM framework to implement hierarchical section-based text mining. Each paper \(p_i\) is divided into \(l_i\) chunks \(\{t_{i1}, t_{i2}, \ldots, t_{il_i}\}\), which are classified into \(s\) predefined section types \(\{section_1, \dots, section_s\}\) by LLMs. 
By generating confidence scores and novelty scores for each section chunk via LLMs, the predicted innovation score \(\hat{\textit{Innovation}}_i\) is obtained as a weighted average of the section novelty scores, with confidence scores serving as weights.
Besides, HSPIM employs QA augmentation before scoring. It optimizes prompts to generate section-specific QA pairs.
A genetic algorithm searches for the optimal prompt combination within \(s\) predefined specific question sets \(Q_1, Q_2, \dots, Q_s\) (one set per section type) and a single common question set \(\mathcal{Q}_c\). Assume that \(t_{ik}\) is classified as \(section_r\) with \(r \in \{1,2,\ldots,s\}\). For each text chunk \(t_{ik}\), the final question \(Q_{ik}\) is formed by concatenating one question from \(Q_r\) and one question from \(\mathcal{Q}_c\). Then \(QA_{ik}\) is incorporated as additional innovation-focused information for \(t_{ik}\) in the scoring prompt, guiding the LLM to generate more meaningful and informative scores. Here, we define our HSPIM as follows:

\begin{definition}[HSPIM] \label{def:HSPIM}
\emph{HSPIM} is a hierarchical text mining framework for measuring scientific paper innovation. Each paper \(p_i\) is segmented into \(l_i\) text chunks \(\{t_{i1}, \ldots, t_{il_i}\}\) and classified into \(s\) predefined section types by LLMs. The predicted innovation score is given by
\begin{align}
\hat{\textit{Innovation}}_i = \frac{\sum_{k=1}^{l_i} \textit{Confidence}_{ik} \,\cdot\, \textit{Novelty}_{ik}}{\sum_{k=1}^{l_i} \textit{Confidence}_{ik}}, \label{eq:predicted_innovation}
\end{align}
where \(\textit{Confidence}_{ik}\) and \(\textit{Novelty}_{ik}\) are LLM-generated values for \(t_{ik}\). 
To enhance scoring, HSPIM uses QA augmentation and a genetic algorithm to optimize question prompts.
It selects an optimal question-prompt combination from \(s\) section-specific question sets and a shared common question set. For each chunk \(t_{ik}\) classified as \(section_r\), the final question \(Q_{ik}\) is formed by combining one question from \(Q_r\) and one from \(\mathcal{Q}_c\). The resulting QA pairs provide additional context to guide the LLM in generating more reliable innovation scores.
\end{definition}

We also introduce HSPIM$^+$ as a norm-based extension of HSPIM. The $p$-norm allows us to combine novelty, contribution, and feasibility in a flexible way. Different values of $p$ let the final score reflect a balanced, cumulative, or dominant effect.

\begin{definition}[HSPIM$^+$ (Norm-based Extension)]
\label{def:HSPIM$^+$}
\emph{HSPIM$^+$} extends the HSPIM framework in Definition \ref{def:HSPIM} by modeling innovation as a composite of three attributes: novelty, contribution, and feasibility. For each chunk \(t_{ik}\), LLMs generate \(\textit{Novelty}_{ik}\), \(\textit{Contribution}_{ik}\), and \(\textit{Feasibility}_{ik}\), along with corresponding confidence scores \(c^{(n)}_{ik}\), \(c^{(d)}_{ik}\), \(c^{(f)}_{ik}\) as weights.

First, for each attribute, a weighted mean is computed:
\begin{align}
\textit{Novelty}_i &= \frac{\sum_{k=1}^{l_i} c^{(n)}_{ik} \cdot \textit{Novelty}_{ik}}{\sum_{k=1}^{l_i} c^{(n)}_{ik}} \\
\textit{Contribution}_i &= \frac{\sum_{k=1}^{l_i} c^{(d)}_{ik} \cdot \textit{Contribution}_{ik}}{\sum_{k=1}^{l_i} c^{(d)}_{ik}} \\
\textit{Feasibility}_i &= \frac{\sum_{k=1}^{l_i} c^{(f)}_{ik} \cdot \textit{Feasibility}_{ik}}{\sum_{k=1}^{l_i} c^{(f)}_{ik}}
\end{align}

Let \(\mathbf{v}_i = [\textit{Novelty}_i,\, \textit{Contribution}_i,\, \textit{Feasibility}_i]\) be the score vector for paper \(p_i\). The overall innovation score is then given by the $p$-norm of this vector, linearly mapped back to the $[1, 5]$ range:
\begin{align}
\label{eq:hspim_plus_norm}
\hat{\textit{Innovation}}_i^{(p)} = \mathcal{M}\left( \left\| \mathbf{v}_i \right\|_p \right)
\end{align}
where $\mathcal{M}(\cdot)$ denotes a linear mapping to $[1, 5]$, and $p \in \{1, 2, \infty\}$ specifies the aggregation norm: $L_1$ (mean), $L_2$ (Euclidean), and $L_\infty$ (max), respectively.
Here, the confidence scores reflect the credibility of the LLMs' evaluation of novelty, contribution, and feasibility.
\end{definition}

\subsection{HSPIM Framework and Workflow}
\label{ssec:HSPIM_workflow}
Fig. \ref{HSPIM_framework} shows the workflow of our hierarchical scientific paper innovation measurement (HSPIM) method.

\subsubsection{Step-1: Section Classification}
% \textbf{Step-1: Section Classification:} 
To effectively analyze long-text scientific papers, we propose section-based method.
First, we use the Science Parse library to extract multiple text chunks from a paper's PDF by identifying section titles. Each text chunk represents a section of the paper. Next, we classify \( k \) text chunks from a paper into nine predefined section types using zero-shot prompting with large language models (LLMs). 
We extend the classic IMRaD structure (Introduction, Methods, Results, and Discussion) \cite{sollaci2004introduction} into nine categories: \textit{Abstract, Introduction, Related Work, Approach/Methodology/Model/Method, Analysis Theory, Experiments, Experiment Analysis, Discussion/Limitations,} and \textit{Conclusion}, where the number of section types is set to \( s = 9 \).
While preserving the IMRaD sections for \textit{Introduction} and \textit{Methods}, we also include the \textit{Abstract}, which is often found at the beginning of academic papers. 
Additionally, we replace \textit{Results} with \textit{Experiments}, a term more commonly used in scientific papers, to improve model compatibility. The \textit{Discussion} section is divided into two distinct parts: \textit{Analysis Theory} for theoretical analysis and \textit{Experiment Analysis} for examining experimental results. Finally, we include \textit{Discussion/Limitations} for reflective insights and \textit{Conclusion} for summarizing key findings.

\begin{figure*}[t]
  \centering
  \includegraphics[width=0.5\textwidth]{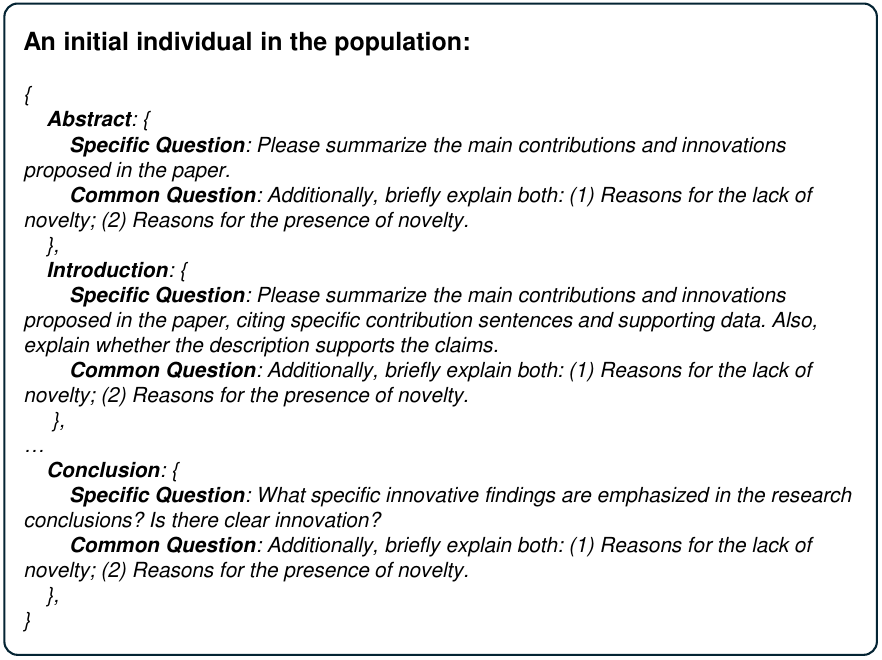}
  \caption{An example of an individual (a question-prompt combination) for multi-prompt optimization. Within an individual, each section has the same common question but a different specific question.}
  \label{initial_individual}
\end{figure*}

\subsubsection{Step-2: Question-Answering Augmentation}
% \textbf{Step-2: Question-Answering Augmentation:} 
To align with the section-based design, we introduce QA augmentation via LLMs to extract innovation-focused information from different sections. QA augmentation means using QA pair as additional section content for subsequent innovation scoring. We design two types of questions: specific questions, which target individual sections, and common questions, which apply to all sections.  

As shown in Fig. \ref{HSPIM_framework}, each section-assigned text chunk selects one question from the corresponding specific question set and one from the common question set to form the final section question:  
\begin{align}
\textit{section question} = \textit{specific question} + \textit{common question}.
\end{align}
For example, the question for the \textit{Abstract} can be composed of ``\textit{Please summarize the main contributions and innovations proposed in the paper.}" from the specific question set of \textit{Abstract}, and ``\textit{Additionally, briefly explain both: (1) Reasons for the lack of novelty; (2) Reasons for the presence of novelty.}" from the common question set.

We note that specific questions are tailored to each section type and only affect the score of the designated section. 
In contrast, common questions ask about innovation-related reasons and influence all sections. Since specific questions are the only prompt differences among sections, they serve as the key mechanism to capture innovation based on section characteristics.

\subsubsection{Step-3: Innovation Scoring}
In Step 3, we build a bridge from local novelty of section content to overall innovation. 
Since innovation is not directly measurable, we provide the original section-based text chunk along with its QA pair as input to LLMs to generate a \textbf{novelty score}. Additionally, we introduce a \textbf{confidence score} to indicate the reliability of the novelty score.
The confidence score serves two purposes. First, it is used to calculate a weighted average of the novelty scores to adjust the importance of each section. Second, it acts as a reliability or soundness measure for the novelty score, providing a further evaluation from novelty to innovation.
The LLM scoring prompt then generates a JSON output containing {\textit{``novelty\_score''}, \textit{``reason''}, \textit{``confidence\_score''}}. \textit{novelty\_score} represents the novelty level of the given section content. 
\textit{reason} explains the rationale behind the novelty score and is used for subsequent qualitative analysis.
\textit{confidence\_score} indicates the reliability of the novelty score.

The formula for calculating the predicted innovation score \(\hat{\textit{Innovation}}\) is given by \(\hat{\textit{Innovation}}_i = \sum_{k=1}^{l_i} \textit{Confidence}_{ik} \cdot \textit{Novelty}_{ik} / \sum_{k=1}^{l_i} \textit{Confidence}_{ik}\), where \( i \) denotes the \( i \)-th paper in the dataset, and \( l_i \) represents the total number of text chunks of the \( i \)-th paper.
Based on innovation theory, we also introduce the innovation attributes of contribution and feasibility, and define HSPIM$^+$ in Definition \ref{def:HSPIM$^+$}.

\subsubsection{Step-4: Evaluate Predicted Innovation}
% \textbf{Step-4: Evaluate Predicted Innovation:} 
We compute the difference between the HSPIM output innovation scores with the ground-truth peer review scores in the dataset.
Our objective is to minimize the RMSE of innovation scores. 
Since the current peer reviews do not provide a direct innovation score, we instead focus on two attribute scoring: originality and soundness/correctness, both rated on a scale of \{1, 2, 3, 4, 5\}.
Originality reflects the novelty, and we use soundness/correctness as a proxy for practical value.
We define innovation as a combination of originality and soundness/correctness.
A paper with a high originality score but a low soundness/correctness score may have methodological flaws or unreliable conclusions, reducing its overall innovation.
Thus, the peer review innovation score \(\textit{Innovation}_i\) for paper \( p_i \) is computed as the average of originality and soundness scores across all its reviews:

\begin{align}
\textit{Innovation}_{i} = \frac{1}{m_i} \sum_{j=1}^{m_i} \frac{\textit{Originality}_{ij} + \textit{Soundness}_{ij}}{2}, \label{eq:label_innovation}
\end{align}
where \( i \) is the paper index, \( m_i \) is the number of reviews for \( p_i \), and \( \textit{Originality}_{ij} \) and \( \textit{Soundness}_{ij} \) are the originality and soundness scores from the \( j \)-th review.

Then we use RMSE to measure the prediction errors:
\begin{equation}
\text{RMSE} = \sqrt{\frac{1}{n} \sum_{i=1}^{n} \left( \textit{Innovation}_{i} - \hat{\textit{Innovation}}_i \right)^2}.
\end{equation}

We define Step-1 to Step-4 as the naive implementation of HSPIM, which includes the complete Paper-to-Sections-to-QAs process and is thus called a hierarchical framework.

\begin{figure*}[ht]
  \centering
  \includegraphics[width=0.8\textwidth]{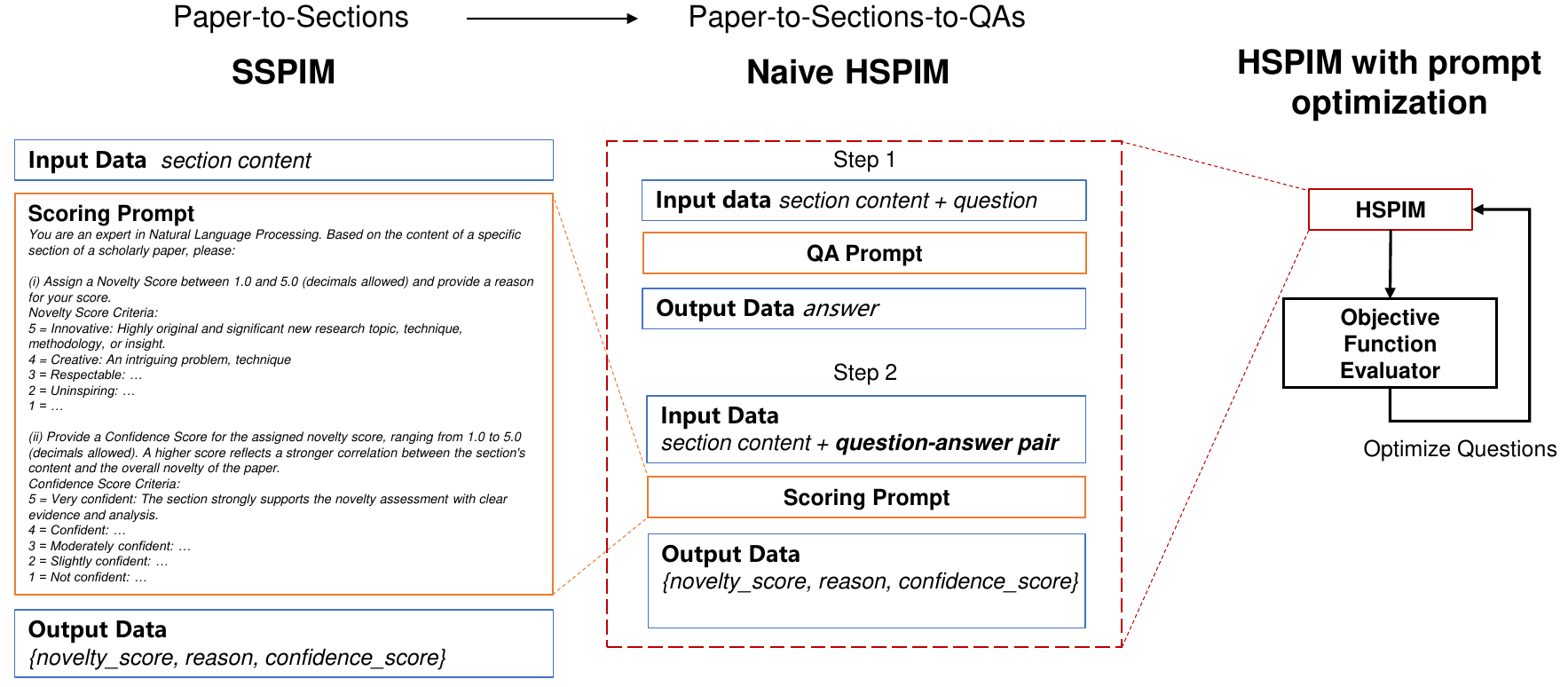}
  \caption{Comparison of section-based scientific paper innovation measurement (SSPIM), naive implement of hierarchical scientific paper innovation measurement (HSPIM) and HSPIM with prompt optimization.}
  \label{all_framework}
\end{figure*}

\begin{figure*}[ht]
    \centering
    \subfloat[\scriptsize Joint Optimization.]{
        \includegraphics[width=0.3\linewidth]{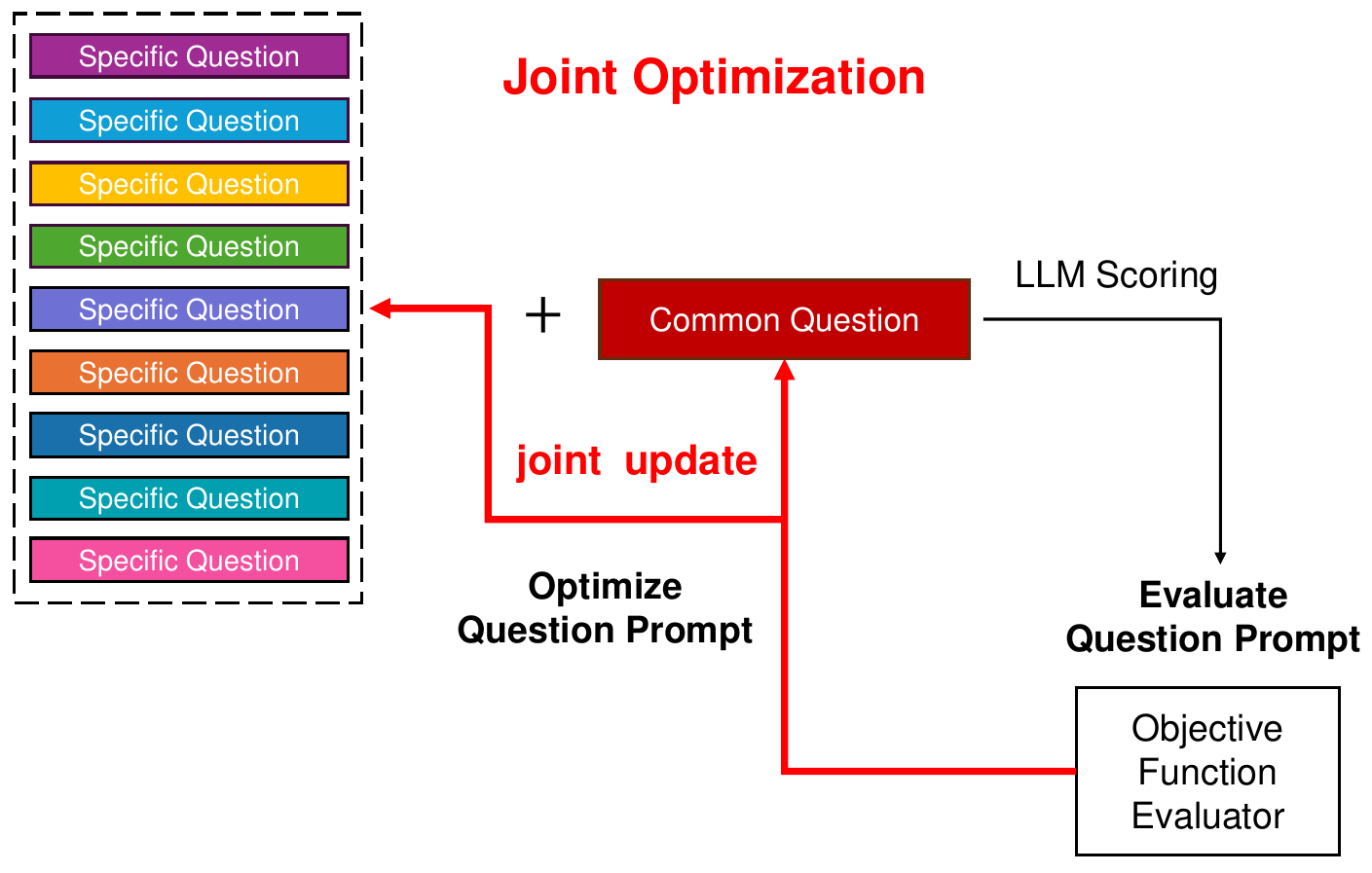}
        \label{fig:joint}
    }
    \subfloat[\scriptsize Two-Step Optimization.]{
        \includegraphics[width=0.3\linewidth]{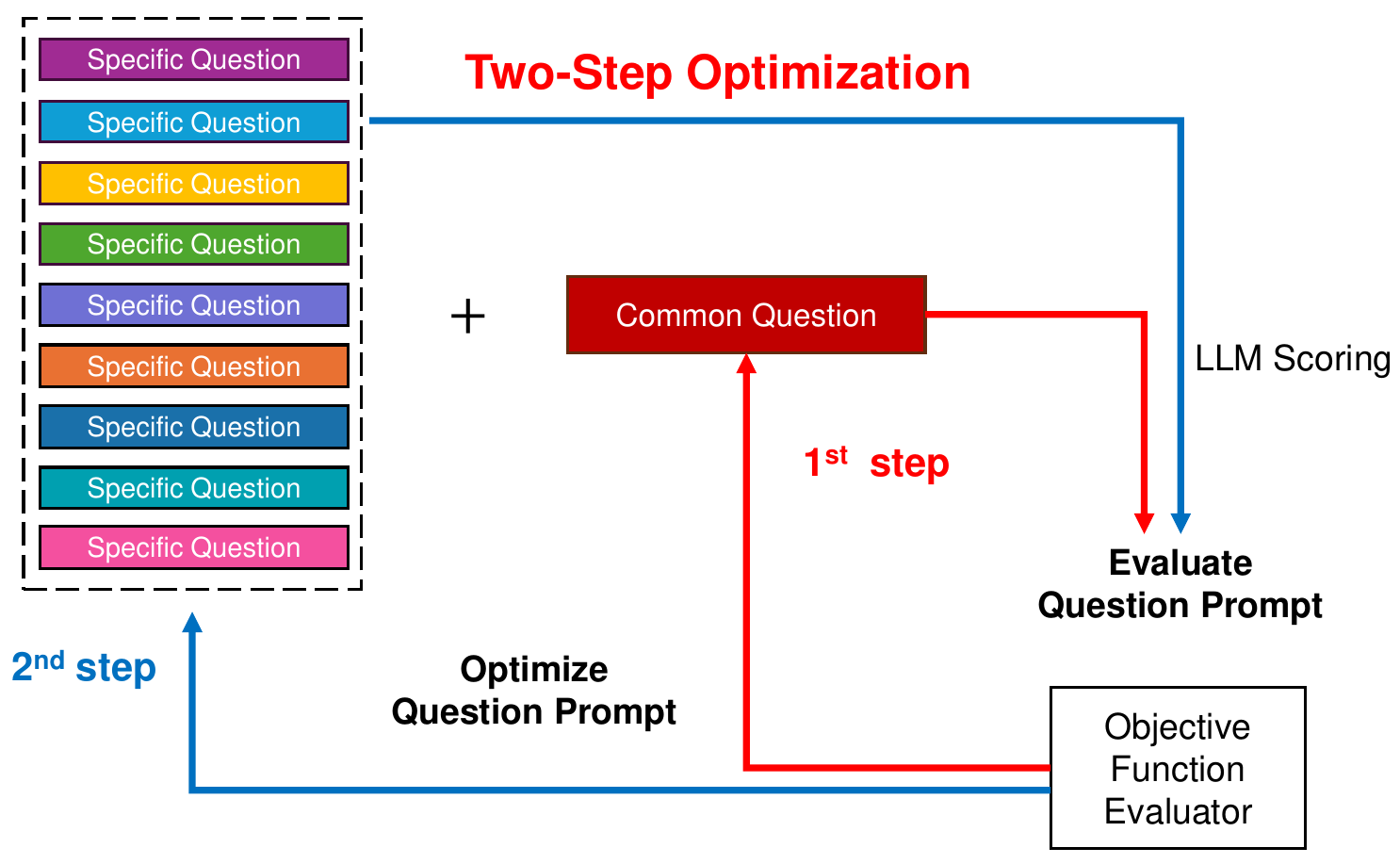}
        \label{fig:two-step}
    }
    \subfloat[\scriptsize Pruning.]{
        \includegraphics[width=0.3\linewidth]{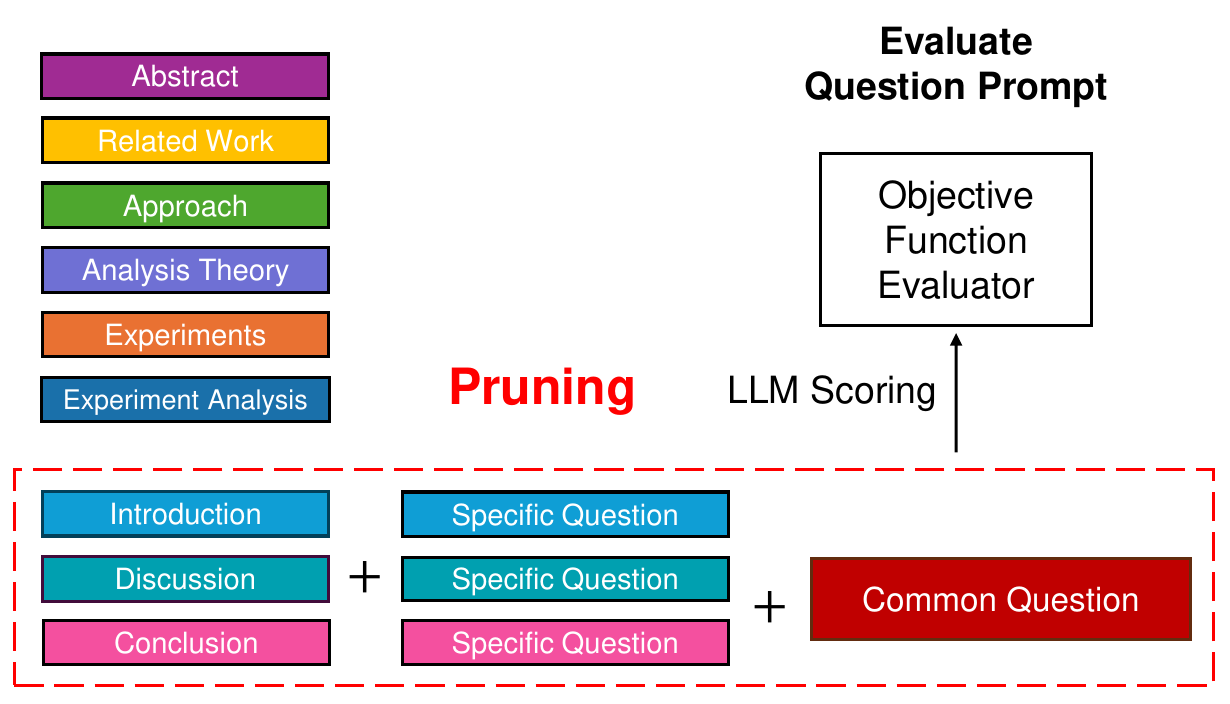}
        \label{fig:pruning}
    }
    \caption{Three types of multi-prompt optimization strategies for two-layer question structure.}
    \label{fig:optimization strategies}
\end{figure*}

\subsubsection{Step-5: Optimize Question Prompt}
\label{sssec:step5_prompt_opt}
% \textbf{Step-5: Optimize Question Prompt:} 
The performance of HSPIM depends on the design of innovation questions. To identify appropriate innovation questions for each section type, we employ prompt optimization to optimize the questions. 
Specifically, we focus on a multi-search-space setting, where multiple text chunks from the same paper are processed with different prompt questions to generate section-based QA pairs.  
Our approach conducts multi-round LLM-based QA generation and scoring. Since text chunks are independent, multiple rounds of QA generation and scoring can run in parallel. 
The final innovation score is obtained by a weighted average of all section scores. 

Previous single-prompt optimization methods mainly update one prompt per round and evaluate its effect on a one-step LLM output. 
These approaches are neither suitable nor efficient for our multi-prompt task.
Our method enables simultaneous QA generation and scoring for multiple text chunks to improve time efficiency. 
We propose a genetic algorithm (GA) optimization approach to jointly optimize multiple question prompts across multiple search spaces.

In our GA implementation, we define an individual as a question-prompt combination. As shown in Fig. \ref{initial_individual}, each individual consists of one common question and \( s \) specific questions. All questions are selected from their corresponding question sets.
We initialize the population by randomly generating \( p \) individuals for each paper.
To compare the effectiveness of different question-prompt combinations (individuals), each paper is evaluated using \( p \) different question combinations, producing \( p \) final innovation scores. The GA prompt optimizer calculates the fitness of all \( p \) individuals using RMSE and ranks them accordingly.
Specifically, the fitness function is defined as the RMSE between the predicted and ground-truth innovation scores on the validation set.
The top-performing individuals are selected as parents for the next generation while keeping elite individuals.

The selected parents then undergo crossover and mutation to generate offspring. 
Crossover selects two parents at random. For each section type, the specific question is inherited from either parent with a 50\% probability. Similarly, the common question is inherited from either parent with the same probability. Mutation modifies the offspring based on the predefined mutation rate \(\mu\). For each section type, the specific question has a probability of \(\mu\) to be replaced with another question from the section question set. Likewise, the common question has a probability of \(\mu\) to be replaced with another question from the common question set. 
These offspring replace the previous individuals for the next iteration, and the process continues until the predefined number of iterations is reached.

\subsection{Framework Comparison}
As described in Section \ref{ssec:HSPIM_workflow}, we progressively construct the hierarchical scientific paper innovation measurement (HSPIM) framework, which consists of Step-1 to Step-5. 

Fig. \ref{all_framework} presents three framework variants: section-based scientific paper innovation measurement (SSPIM), a naive implementation of HSPIM, and HSPIM with prompt optimization.
SSPIM includes Step-1, Step-3, and Step-4, and it uses a single-turn prompting approach to achieve the Paper-to-Sections framework. 
Naive HSPIM covers Step-1 to Step-4 and employs a two-turn dialogue, implementing a hierarchical Paper-to-Sections-to-QAs framework.
HSPIM with prompt optimization incorporates all five steps and iteratively optimizes prompts through multiple-turn dialogues.
The novelty and confidence scoring criteria in our scoring prompt follow the ACL-2018 reviewer guidelines\footnote{https://acl2018.org/downloads/acl\_2018\_review\_form.html}. The QA prompt is adapted from the QAG prompt in the QAEA-DR framework \cite{tan2024qaea}.

\subsection{Multi-prompt Optimization Strategies in HSPIM}
\label{sec:Optimization_Strategies}
Fig. \ref{initial_individual} shows an individual (a question-prompt combination) for multi-prompt optimization.
Each prompt combination follows a two-layer question structure, which includes \(s\) specific questions and a common question, where \(s\) denotes the number of predefined section types.

In multi-prompt optimization, question prompts evolve similarly to genetic algorithms, where specific questions function analogously to structural genes affecting distinct traits, and common questions act as regulatory genes influencing broader characteristics \cite{jacob1961genetic,forrest1993genetic}. Each specific question targets precisely one section type, akin to a mutation in a structural gene causing targeted phenotypic changes. In contrast, the common question governs global attributes of innovation evaluation, mirroring the role of regulatory genes in coordinating the expression of multiple genes simultaneously \cite{jacob1961genetic}.

Under the standard genetic algorithm (GA) with this two-layer question structure, we explore optimization strategies including \textbf{joint optimization}, \textbf{two-step optimization}, and \textbf{pruning}, as illustrated in Fig. \ref{fig:optimization strategies}.

\textbf{Joint Optimization:} Fig. \ref{fig:optimization strategies} (a) shows that joint optimization jointly searches for better common and specific question prompts in each GA iteration and updates them as the offspring for the next evaluation round.
It serves as the default optimization strategy, enabling rapid updates of all \(s+1\) question prompts. Hereafter, we denote the HSPIM framework with joint optimization as HSPIM\textsubscript{joint}.

\textbf{Two-step Optimization:} Fig. \ref{fig:optimization strategies} (b) illustrates a two-step optimization process. First, update the common question and complete its optimization. Then, proceed to update \(s\) specific questions. In the following sections, we refer to the HSPIM framework with this two-step optimization as HSPIM\textsubscript{two-step}.

\textbf{Pruning:} From Fig. \ref{fig:optimization strategies} (c), pruning selects a subset of \(s'\) section types (\( s' < s \)) instead of using all \( s \) section types when computing the weighted novelty as the innovation score. 
We traverse all combinations of \(s'\) section types in the training set, select the best-performing one, and apply it to the test set.
In the figure, only three section types are kept for the final calculation (\(s' = 3\)). This strategy helps to identify the most important sections in the dataset and to show their impact on the innovation score. It also helps to reduce the cost in real-world applications.

\section{Theoretical Analysis of HSPIM}
\label{sec:theory}
In this section, we provide a theoretical analysis of our HSPIM framework, focusing on the unbiasedness of the weighted scoring function and the convergence properties of the genetic algorithm-based prompt optimization. Our notation and definitions follow Section \ref{sec:notations}, but we make slight modifications to accommodate the revised bounded interval and unbiasedness assumptions for both innovation attributes and \(\textit{Confidence}_{ik}\).

\subsection{Preliminaries and Assumptions}

Recall that each paper \(p_i\) is divided into \(l_i\) text chunks, \(\{t_{i1}, t_{i2}, \dots, t_{il_i}\}\), and each chunk \(t_{ik}\) is classified by an LLM into one of \(s\) section types. After section-specific QA augmentation, the LLM outputs two scores for each chunk \(t_{ik}\):
\(\textit{Novelty}_{ik}\) and \(\textit{Confidence}_{ik}\). We then form the predicted innovation \(\hat{\textit{Innovation}}_i\) via Eq.~\eqref{eq:predicted_innovation}, while the ground-truth label is \(\textit{Innovation}_i\) from Eq.~\eqref{eq:label_innovation}. Let us define the prediction error for paper \(p_i\) as \(\varepsilon_i = \hat{\textit{Innovation}}_i - \textit{Innovation}_i\).

We make the following assumptions:

\begin{itemize}
    \item[(A1)] \textbf{Bounded Ouputs:} There exist constants \(0 < L < M\) such that \(\textit{Novelty}_{ik}, \textit{Confidence}_{ik} \in [L, M]\) for all \(i,k\). This ensures both scores remain strictly above zero and below \(M\).

    % \item[(A2)] \textbf{Section-Based Unbiasedness \emph{(in expectation)}:} 
    % For each chunk \(t_{ik}\), the random variables \(\textit{Novelty}_{ik}\) and \(\textit{Confidence}_{ik}\) satisfy \(\mathbb{E}[\textit{Novelty}_{ik}] = n_{ik}^*,\;\mathbb{E}[\textit{Confidence}_{ik}] = c_{ik}^*,\) and \(\textit{Novelty}_{ik}\) is independent of \(\textit{Confidence}_{ik}\).
    % Here, \(n_{ik}^*\) and \(c_{ik}^*\) represent ``true'' local novelty and confidence values, respectively.
    \item[(A2)] \textbf{Section-Based Unbiasedness \emph{(in expectation)}:} 
    For each chunk \(t_{ik}\), the random variables \(\textit{Novelty}_{ik}\) and \(\textit{Confidence}_{ik}\) satisfy \(\mathbb{E}[\textit{Novelty}_{ik}]\allowbreak = n_{ik}^*,\;\mathbb{E}[\textit{Confidence}_{ik}] = c_{ik}^*.\) \(\textit{Novelty}_{ik}\) is independent of \(\textit{Confidence}_{ik}\).
    Here, \(n_{ik}^*\) and \(c_{ik}^*\) represent ``true'' local novelty and confidence values, respectively.

    \item[(A3)] \textbf{Noise Independence Across Section-Chunks:}
    Random variations in \((\textit{Novelty}_{ik}, \textit{Confidence}_{ik})\) across different chunks are independent or exhibit only negligible dependence.
\end{itemize}

These assumptions provide a simplified yet tractable basis for analyzing HSPIM.

\subsection{Convergence of the Weighted Scoring Mechanism}
\label{ssec:weighted_scoring_convergence}
Under (A1)--(A3), the weighted scoring function in \eqref{eq:predicted_innovation} converges in expectation to a ``true'' innovation measure. We outline the key result below, applying the same logic to both \(\textit{Novelty}_{ik}\) and \(\textit{Confidence}_{ik}\) but omitting redundant proofs for brevity.

\begin{lemma}[Section-Based Bias and Variance]
\label{lemma:bias_variance}
For chunk \(t_{ik}\), assume (A1)--(A3). Then
\(\mathbb{E}[\textit{Novelty}_{ik}] = n_{ik}^*,\;\mathbb{E}[\textit{Confidence}_{ik}] = c_{ik}^*.\)
Both \(\textit{Novelty}_{ik}\) and \(\textit{Confidence}_{ik}\) have finite variances since they lie in \([L,M]\).
\end{lemma}

\begin{theorem}[Unbiasedness of the weighted scoring function]
\label{thm:unbiased_aggregator}
Let \(\hat{\textit{Innovation}}_i\) be as defined in Eq.~\eqref{eq:predicted_innovation}. Under (A1)--(A3) and Lemma~\ref{lemma:bias_variance}, we have
\begin{equation}
\begin{aligned}
\mathbb{E}[\hat{\textit{Innovation}}_i] 
&= \frac{\sum_{k=1}^{l_i} \mathbb{E}[\textit{Confidence}_{ik}] \cdot \mathbb{E}[\textit{Novelty}_{ik}]}{\sum_{k=1}^{l_i} \mathbb{E}[\textit{Confidence}_{ik}]} \\
&= \frac{\sum_{k=1}^{l_i} c_{ik}^*\,n_{ik}^*}{\sum_{k=1}^{l_i} c_{ik}^*}.
\end{aligned}
\end{equation}
Hence, \(\hat{\textit{Innovation}}_i\) is an unbiased weighted scoring of local novelty values in expectation.
\end{theorem}

\begin{proof}[Proof Sketch]
Boundedness (A1) ensures finite variance. Under mild regularity conditions and the independence assumption in (A2), we can exchange the summation and division as follows:
\begin{equation}
\begin{aligned}
\mathbb{E}\bigl[\hat{\textit{Innovation}}_i\bigr]
&= \mathbb{E}\Bigl[
     \frac{\sum_{k=1}^{l_i} \textit{Confidence}_{ik}\,\textit{Novelty}_{ik}}
          {\sum_{k=1}^{l_i} \textit{Confidence}_{ik}}
   \Bigr] \\
&= \frac{\sum_{k=1}^{l_i}\mathbb{E}\bigl[\textit{Confidence}_{ik}\,\textit{Novelty}_{ik}\bigr]}
     {\sum_{k=1}^{l_i}\mathbb{E}\bigl[\textit{Confidence}_{ik}\bigr]}.
\end{aligned}
\end{equation}
% Since \(\textit{Novelty}_{ik}\) and \(\textit{Confidence}_{ik}\) are independent, we have 
% \(\mathbb{E}[\textit{Confidence}_{ik}\,\textit{Novelty}_{ik}] = \mathbb{E}[\textit{Confidence}_{ik}]\,\mathbb{E}[\textit{Novelty}_{ik}] = c_{ik}^*\,n_{ik}^*\).
Since \(\textit{Novelty}_{ik}\) and \(\textit{Confidence}_{ik}\) are independent, we have:
\begin{equation}
\begin{aligned}
\mathbb{E}[\textit{Confidence}_{ik}\,\textit{Novelty}_{ik}] 
&= \mathbb{E}[\textit{Confidence}_{ik}]\,\mathbb{E}[\textit{Novelty}_{ik}] \\
&= c_{ik}^*\,n_{ik}^*.
\end{aligned}
\end{equation}
Substituting completes the proof.
\end{proof}

\noindent
\textbf{Discussion.} 
If local chunk-level novelty and confidence are unbiased in expectation, the final weighted average preserves unbiasedness for overall innovation. Each chunk focuses on a narrower attribute of novelty, weighted by its confidence.

\subsection{Convergence of GA-Based Prompt Optimization}

We now analyze the genetic algorithm (GA) used to optimize section-specific and common questions.

\begin{definition}[Discrete Prompt Combination]
\label{def:prompt_comb}
Let \(Q_r\) be the set of candidate \emph{specific} questions for section type \(r\), with \(|Q_r| = N_r\), and \(\mathcal{Q}_c\) be the set of candidate \emph{common} questions with \(|\mathcal{Q}_c|=N_c\). A \emph{prompt combination} (or \emph{individual}) is \((q_c, q_1,\ldots,q_s)\) where \(q_c \in \mathcal{Q}_c\) and \(q_r \in Q_r\). The finite search space has size
\(\bigl|\mathcal{Q}_c\bigr|\,\times\,\prod_{r=1}^s |Q_r| = N_c \,\times\,\prod_{r=1}^s N_r.\)
\end{definition}

Given a prompt combination \(x\), HSPIM generates predicted innovation scores \(\hat{\textit{Innovation}}_i(x)\) for each paper \(p_i\). The prediction discrepancy for paper \(p_i\) is defined as \(\varepsilon_i(x)=\hat{\textit{Innovation}}_i(x)-\textit{Innovation}_i.\)
The RMSE over the dataset is computed as
\begin{equation}
\text{RMSE}(x)=\sqrt{\frac{1}{n}\sum_{i=1}^{n}\varepsilon_i(x)^2}.
\end{equation}
The optimal prompt combination is
\begin{equation}
x^*=\arg\min_{x}\,\text{RMSE}(x).
\end{equation}

\begin{theorem}[GA Convergence on Finite Space]
\label{thm:ga_convergence}
Consider an elitist genetic algorithm (GA) with nonzero mutation rate \(\mu > 0\). As the generation number \(G\to\infty\), the probability that the GA population contains the optimal prompt combination \(x^*\) converges to 1:
\begin{equation}
\lim_{G\to\infty}\mathbb{P}\bigl[x^*\text{ is in the population at generation }G\bigr]=1.
\end{equation}
\end{theorem}

\begin{proof}[Proof Sketch]
The elitist GA with nonzero mutation probability can be represented as a finite-state, irreducible, and aperiodic Markov chain \cite{davis1993markov}. Nonzero mutation ensures ergodicity, allowing exploration of the entire finite search space, including the global optimum \(x^*\). Elitism guarantees that once \(x^*\) is found, it persists in subsequent generations. 
According to Theorem 2 in \cite{rudolph1996convergence}, which establishes the convergence conditions for evolutionary algorithms, the algorithm converges almost surely to the global optimum.
\end{proof}

\noindent
\textbf{Implications.}
A sufficiently long GA run eventually finds and retains the global optimum \(x^*\). In practice, we stop once performance saturates.

\subsection{Overall Effectiveness and Convergence Discussion}

Combining Theorem~\ref{thm:unbiased_aggregator} and Theorem~\ref{thm:ga_convergence} shows:

\begin{itemize}
    \item \textbf{Local Novelty to Global Innovation:}
    By weighting section-based novelty with confidence, HSPIM yields an unbiased full paper innovation estimate, given unbiasedness in expectation for both scores.
    \item \textbf{Guaranteed Prompt Optimization:}
    Elitist GA on a finite discrete space converges to a global minimum with probability 1, ensuring the best question-prompt combination.
    \item \textbf{Practical Constraints:}
    LLM evaluations incur cost, so GA often stops early. Even partial runs can significantly reduce RMSE compared to single-turn QA method.
\end{itemize}

These results justify why a hierarchical paper-to-sections-to-QAs design and a GA multi-prompt optimizer can effectively capture global innovation via local novelty scoring.

\subsection{Theoretical Justification of HSPIM$^+$}
\label{ssec:HSPIMplus_effectiveness}

To capture the composite nature of innovation, HSPIM$^+$ models it as a function of three core attributes: novelty, contribution, and feasibility. The mathematical structure is:
\begin{equation}
\label{eq:innovation_general_norm}
\textit{Innovation} = f(\textit{Novelty},\, \textit{Contribution},\, \textit{Feasibility}),
\end{equation}
where $f$ is the mapping defined by a $p$-norm with linear normalization (see Definition~\ref{def:HSPIM$^+$}).

We make the following assumptions for all papers $p_i$, chunks $t_{ik}$, and innovation attributes:

\begin{itemize}
    \item[(B1)] All innovation attribute scores (novelty, contribution, feasibility) and their confidence scores are bounded and strictly positive.
    \item[(B2)] Each attribute score is unbiased in expectation and independent of its corresponding confidence score.
    \item[(B3)] Innovation attribute scores and confidence scores are independent across section chunks.
    \item[(B4)] \textbf{Idealized Independence:} For each chunk, the innovation attribute scores and their confidence scores are mutually independent across novelty, contribution, and feasibility; that is, there is no dependence among the three attributes in either scoring or confidence assignment.
\end{itemize}

\noindent
\textbf{Remark.} Assumption (B4) idealizes the three attributes as fully independent for analytical tractability. In practice, some correlation may exist (e.g., high novelty may reduce feasibility), but we follow standard practice in multi-attribute aggregation.

\begin{theorem}[Unbiasedness of Norm-Based HSPIM$^+$]
\label{thm:unbiased_HSPIMplus}
Let $\hat{\textit{Innovation}}_i^{(p)}$ be the innovation score for paper $p_i$ computed by HSPIM$^+$ using the $p$-norm aggregation and linear normalization, as defined in Definition~\ref{def:HSPIM$^+$}, under assumptions (B1)--(B4).

Then, the predicted vector $\mathbf{v}_i = [\textit{Novelty}_i,\, \textit{Contribution}_i,\, \textit{Feasibility}_i]$ satisfies:
\begin{align}
\mathbb{E}[\textit{Novelty}_i] &= n_i^*, \\
\mathbb{E}[\textit{Contribution}_i] &= d_i^*, \\
\mathbb{E}[\textit{Feasibility}_i] &= f_i^*,
\end{align}
where $n_i^*$, $d_i^*$, $f_i^*$ are the true paper-level values for the three attributes, computed as the confidence-weighted means of the corresponding true section-level values.

Moreover, the overall innovation score $\hat{\textit{Innovation}}_i^{(p)}$ is an unbiased, monotonic transformation of these means via the $p$-norm and normalization:
\begin{equation}
\mathbb{E}[\hat{\textit{Innovation}}_i^{(p)}] = 1 + 4 \cdot \frac{ \| [n_i^*, d_i^*, f_i^*] \|_p - x_{\min} }{ x_{\max} - x_{\min} },
\end{equation}
where $x_{\min}$ and $x_{\max}$ are the theoretical minimum and maximum values of the $p$-norm in $[1,5]^3$.
\end{theorem}

\begin{proof}
By Theorem \ref{thm:unbiased_aggregator}, under (B2)--(B3), each confidence-weighted mean (i.e., $\textit{Novelty}_i$, $\textit{Contribution}_i$, $\textit{Feasibility}_i$) is an unbiased estimator of its true paper-level mean ($n_i^*$, $d_i^*$, $f_i^*$).  

Assumption (B4) further ensures that these three means are mutually independent. Since the $p$-norm and normalization are deterministic and monotonic, the final innovation score $\hat{\textit{Innovation}}_i^{(p)}$ is thus an unbiased and monotonic transformation of the true means. 
\end{proof}

\noindent
\textbf{Summary.} 
HSPIM$^+$ provides an unbiased estimator for composite innovation in expectation, under the independence assumptions.

\section{Experiment}
In this section, we first describe the datasets, evaluation metrics, baselines, and experimental settings. Then, we present the main results and semantic textual similarity analysis.

\begin{table}[ht]
\centering
\caption{Statistics of peer review datasets.}
\label{tab:dataset_stats}
% \renewcommand{\arraystretch}{1.2}
% \resizebox{\textwidth}{!}{%
\scriptsize
\begin{tabular}{l|c|c|c}
\hline
Datasets & Domain & \# Papers & \# Reviews \\
\hline
\hline
ICLR-2017  & ML/AI & 427 & 1304 \\
CONLL-2016 & NLP/CL & 22  & 39   \\
ACL-2017   & NLP/CL & 136 & 272  \\
COLING-2020   & NLP/CL & 89 & 112  \\
\hline
\end{tabular}
% }
\end{table}

\subsection{Datasets}
We choose scientific peer review datasets from ACL-2017, CoNLL-2016, and ICLR-2017 in PeerRead \cite{kang2018dataset}, and the COLING-2020 dataset from NLPeer \cite{dycke2023nlpeer}. Table \ref{tab:dataset_stats} shows that ICLR-2017 belongs to the ML/AI domain, while the others are from the NLP/CL domain.
As describe in Step-4 of Section \ref{ssec:HSPIM_workflow}, these datasets provide numeric attribute scores of originality and soundness/correctness to form the ground-truth innovation score.
Note that originality scores are mainly available in NLP-related conferences. 
For ICLR-2017, originality scores were manually annotated by the PeerRead authors and validated through the consistency re-annotation experiment.
All labeled scores range from 1 to 5 (integers only).

\subsection{Metrics}
Our task is to predict innovation scores for \(n\) papers in dataset \(\mathcal{C}\). The objective is to minimize the error between the predicted scores \(\hat{y}_i\) and the ground-truth scores \(y_i\).
We use root mean square error (RMSE) and mean absolute error (MAE) as the main evaluation metrics:

\begin{equation}
\text{RMSE} = \sqrt{\tfrac{1}{n} \sum_{i=1}^{n} (\hat{y}_i - y_i)^2}, \quad
\text{MAE} = \tfrac{1}{n} \sum_{i=1}^{n} |\hat{y}_i - y_i|.
\end{equation}

To measure the semantic similarity between the LLM-generated scoring reasons and the peer review comments, we compute cosine similarity and BERTScore \cite{zhangbertscore}. 
These metrics help to evaluate the effectiveness and interpretability of HSPIM through generated explanations.

\begin{table*}[ht!]
\centering
\caption{
Performance comparison among zero-shot HSPIM, supervised deep learning models, and other zero-shot LLMs. The bold values represent the best results.
}
\label{tab:unfair_comparison}
\scriptsize
% \renewcommand{\arraystretch}{0.7}
% \resizebox{\textwidth}{!}{%
\begin{tabular}{c|cccccccc|cc}
\toprule
\multicolumn{1}{c}{Datasets} 
    & \multicolumn{2}{|c}{ICLR-2017} 
    & \multicolumn{2}{c}{CoNLL-2016}
    & \multicolumn{2}{c}{ACL-2017}
    & \multicolumn{2}{c}{COLING-2020}
    & \multicolumn{2}{|c}{\textit{Average}} \\
\midrule
\multicolumn{1}{c|}{Metric}
    & RMSE & MAE
    & RMSE & MAE
    & RMSE & MAE
    & RMSE & MAE
    & RMSE & MAE \\
\midrule
\multicolumn{11}{c}{\textit{Supervised Models}} \\
CNN~\cite{lecun1989backpropagation}
    & 1.2753 & 1.0387
    & 0.7237 & 0.6269
    & 0.7629 & 0.6718
    & 1.1340 & 0.9467
    & 0.9740 & 0.8210 \\

LSTM~\cite{hochreiter1997long}
    & 1.0850 & 0.9124
    & 0.3599 & 0.3125
    & 0.5673 & 0.4270
    & 0.6149 & 0.5334
    & 0.6568 & 0.5463 \\

Transformer~\cite{vaswani2017attention}
    & 0.9915 & 0.8043
    & 0.3911 & 0.3672
    & \textbf{0.5551} & \textbf{0.4210}
    & 0.5844 & 0.4563
    & 0.6305 & 0.5122 \\
\midrule
\multicolumn{11}{c}{\textit{Zero-Shot LLMs}} \\
SEA-E (Mistral-7B)~\cite{yu2024automated}
    & 1.2502 & 1.1373
    & 0.3758 & 0.3611
    & 1.7627 & 1.5000
    & 1.2939 & 1.1154
    & 1.1707 & 1.0284 \\

DeepSeek-V3~\cite{liu2024deepseek}
    & 1.0333 & 0.7647
    & 0.6473 & 0.6389
    & 0.7856 & 0.7000
    & 0.6015 & 0.4808
    & 0.9162 & 0.8032 \\

GPT-4o mini~\cite{achiam2023gpt}
    & 1.1044 & 0.8235
    & 1.0285 & 0.9722
    & 0.6622 & 0.6286
    & 0.8695 & 0.7885
    & 0.7669 & 0.6461 \\
\midrule
DeepSeek-HSPIM\textsubscript{two-step}
    & 0.8629 & 0.7194
    & 0.0907 & 0.0617
    & 1.0964 & 0.8738
    & 0.5955 & 0.4616
    & 0.6614 & 0.5291 \\

DeepSeek-HSPIM\textsubscript{joint}
    & 0.8283 & 0.6966
    & \textbf{0.0638} & \textbf{0.0560}
    & 1.0887 & 0.8572
    & 0.5202 & 0.3933
    & 0.6253 & 0.5008 \\

DeepSeek-HSPIM\textsubscript{pruning=3}
    & 0.8258 & 0.6724
    & 0.0891 & 0.0569
    & 0.9634 & 0.7522
    & 0.4879 & 0.3911
    & \textbf{0.5915} & \textbf{0.4681} \\

GPT-HSPIM\textsubscript{two-step}
    & 0.8685 & 0.7324
    & 0.1776 & 0.0550
    & 1.1052 & 0.8905
    & 0.5087 & 0.4082
    & 0.6650 & 0.5215 \\

GPT-HSPIM\textsubscript{joint}
    & 0.8474 & 0.7160
    & 0.1973 & 0.1736
    & 1.1205 & 0.8897
    & 0.5075 & 0.4129
    & 0.6682 & 0.5481 \\

GPT-HSPIM\textsubscript{pruning=3}
    & \textbf{0.8190} & \textbf{0.6667}
    & 0.2369 & 0.1806
    & 1.0329 & 0.8036
    & \textbf{0.4475} & \textbf{0.3750}
    & 0.6341 & 0.5064 \\

\bottomrule
\end{tabular}%
% }
% \vspace{3pt}
% \raggedright
% The bold values represent the best results.
\end{table*}

\subsection{Baselines}
In the experiments, we compare deep learning (DL) models including CNN \cite{lecun1989backpropagation}, LSTM \cite{hochreiter1997long}, and Transformer \cite{vaswani2017attention}. These baseline DL models perform supervised regression on each dataset. To provide a fairer comparison with zero-shot methods, we also train the DL models on COLING-2020 and evaluate them on the other three datasets. 

We also include baseline LLMs for comparison: DeepSeek-V3 (671B) \cite{liu2024deepseek}, GPT-4o mini \cite{achiam2023gpt}, and SEA-E \cite{yu2024automated}, a peer review generator fine-tuned from Mistral-7B. Baseline LLMs use one-step zero-shot prompt to generate scores directly from the full paper.

In our HSPIM framework, we mainly use DeepSeek-V3 and GPT-4o mini as LLM generators and scorers, denoted as DeepSeek-HSPIM and GPT-HSPIM, respectively. 
We also conduct internal comparisons among three HSPIM variants (SSPIM, naive HSPIM, and HSPIM with prompt optimization) to demonstrate the effectiveness of the hierarchical design and prompt optimization.

\subsection{Hyperparameters and Experimental Setting}
\textbf{LLM Hyperparameters:}
We follow the hyperparameter settings used in prompt optimization methods such as EVOPROMPT. For LLM QA generation, we set the temperature to 1.0 to encourage diverse and informative responses based on the input question and section content.
For LLM scoring, we use a temperature of 0 to reduce randomness and ensure more stable predictions.
We use the APIs of DeepSeek-V3 and GPT-4o mini in our experiments.

For model comparison, we use the peer review generator SEA-E \cite{yu2024automated}, which is fine-tuned on Mistral-7B-Instruct-v0.2. We set the temperature to 0 to generate peer reviews that include soundness and originality scores, which are extracted to compute innovation scores.

\textbf{Genetic Algorithm Hyperparameters:} To test the effectiveness under resource-limited scenarios, we set the number of iterations \( I = 5 \) and the population size \(P\) to 10. 
The predefined specific and common question sets each contain 11 candidates. The initial question-prompt combination used in naive HSPIM is selected from these sets.
We set the mutation rate to 10\%, and the number of elite individuals to \(\max\{1,\, 0.2 \times P\}\). The fitness function uses RMSE to evaluate prediction performance.

\textbf{Supervised Baseline Settings:} 
We use the PeerRead setup for deep learning models and take the first 1000 tokens of each paper as input.
We select only the originality and soundness/correctness attribute scores for training.
The training epochs are set to 10 for CNN, LSTM, and Transformer. 
For CNN and LSTM, we set the number of layers to 2, the dropout rate to 0.5, the learning rate to 0.001, and the batch size to 64. For the Transformer encoder, we use 12 attention heads, a hidden size of 768, a learning rate of 2e-5, and a batch size of 8. All results are averaged over three runs.

\begin{table}[ht!]
\centering
\caption{
Performance comparison on ICLR-2017, CoNLL-2016, and ACL-2017: zero-shot HSPIM vs. DL models trained on COLING-2020. The bold values represent the best results.
}
\label{tab:fair_comparison}
\scriptsize
\setlength{\tabcolsep}{0.4pt}
% \resizebox{\textwidth}{!}{%
\begin{tabular}{c|cc|cc|cc}
\toprule
\multicolumn{1}{c|}{Datasets}
& \multicolumn{2}{c|}{ICLR-2017}
& \multicolumn{2}{c|}{CoNLL-2016}
& \multicolumn{2}{c}{ACL-2017} \\
\hline
\multicolumn{1}{c|}{Metrics}
& RMSE & MAE
& RMSE & MAE
& RMSE & MAE \\
\midrule
CNN \cite{lecun1989backpropagation}
    & 1.0172 & 0.7925
    & 0.6973 & 0.5296
    & 1.6393 & 1.4149 \\
LSTM \cite{hochreiter1997long}
    & 1.1350 & 1.0049
    & 0.9801 & 0.8417
    & 1.1122 & 1.0281 \\
Transformer \cite{vaswani2017attention}
    & 0.9135 & 0.7470
    & 0.8719 & 0.7718
    & 1.0225 & 0.9047 \\
\midrule
DeepSeek-HSPIM\textsubscript{two-step}
    & 0.8629 & 0.7194
    & 0.0907 & 0.0617
    & 1.0964 & 0.8738 \\
DeepSeek-HSPIM\textsubscript{joint}
    & 0.8283 & 0.6966
    & \textbf{0.0638} & \textbf{0.0560}
    & 1.0887 & 0.8572 \\
DeepSeek-HSPIM\textsubscript{pruning=3}
    & 0.8258 & 0.6724
    & 0.0891 & 0.0569
    & \textbf{0.9634} & \textbf{0.7522} \\
GPT-HSPIM\textsubscript{two-step}
    & 0.8685 & 0.7324
    & 0.1776 & 0.0550
    & 1.1052 & 0.8905 \\
GPT-HSPIM\textsubscript{joint}
    & 0.8474 & 0.7160
    & 0.1973 & 0.1736
    & 1.1205 & 0.8897 \\
GPT-HSPIM\textsubscript{pruning=3}
    & \textbf{0.8190} & \textbf{0.6667}
    & 0.2369 & 0.1806
    & 1.0329 & 0.8036 \\
\bottomrule
\end{tabular}
% }
% \\[3pt]
% \raggedright
% The bold values represent the best results.
\end{table}

\subsection{Main Results}
This section presents the main experimental results. We compare our HSPIM framework with baseline models on four datasets using RMSE and MAE as evaluation metrics.
In the following results, HSPIM\textsubscript{two-step} and HSPIM\textsubscript{joint} denote the two-step and joint optimization for HSPIM, respectively.
HSPIM\textsubscript{pruning=3} selects three section types (\(s' = 3\)) to compute the innovation scores, following Section \ref{sec:Optimization_Strategies}.
We use HSPIM\textsubscript{joint} as the default framework to represent HSPIM with prompt optimization in the following experiments.

\subsubsection{Zero-shot HSPIM vs. Supervised DL Models and Zero-shot LLMs}
Table \ref{tab:unfair_comparison} shows that HSPIM framework outperforms both supervised baseline models and zero-shot baseline LLMs in terms of RMSE and MAE. Compared to one-step zero-shot prompting on full papers, HSPIM based on paper-to-sections-to-QAs decomposition reduces the average RMSE by 35.4\% (from 0.916 to 0.592) for DeepSeek and by 17.3\% (from 0.770 to 0.634) for GPT. 
Besides, our HSPIM outperforms all supervised baselines. This proves that the zero-shot approach can generalize well. For instance, DeepSeek-HSPIM\textsubscript{pruning=3} achieves a 6.2\% performance improvement over the supervised Transformer.

\begin{figure}[ht!]
  \centering
  \includegraphics[width=0.45\textwidth]{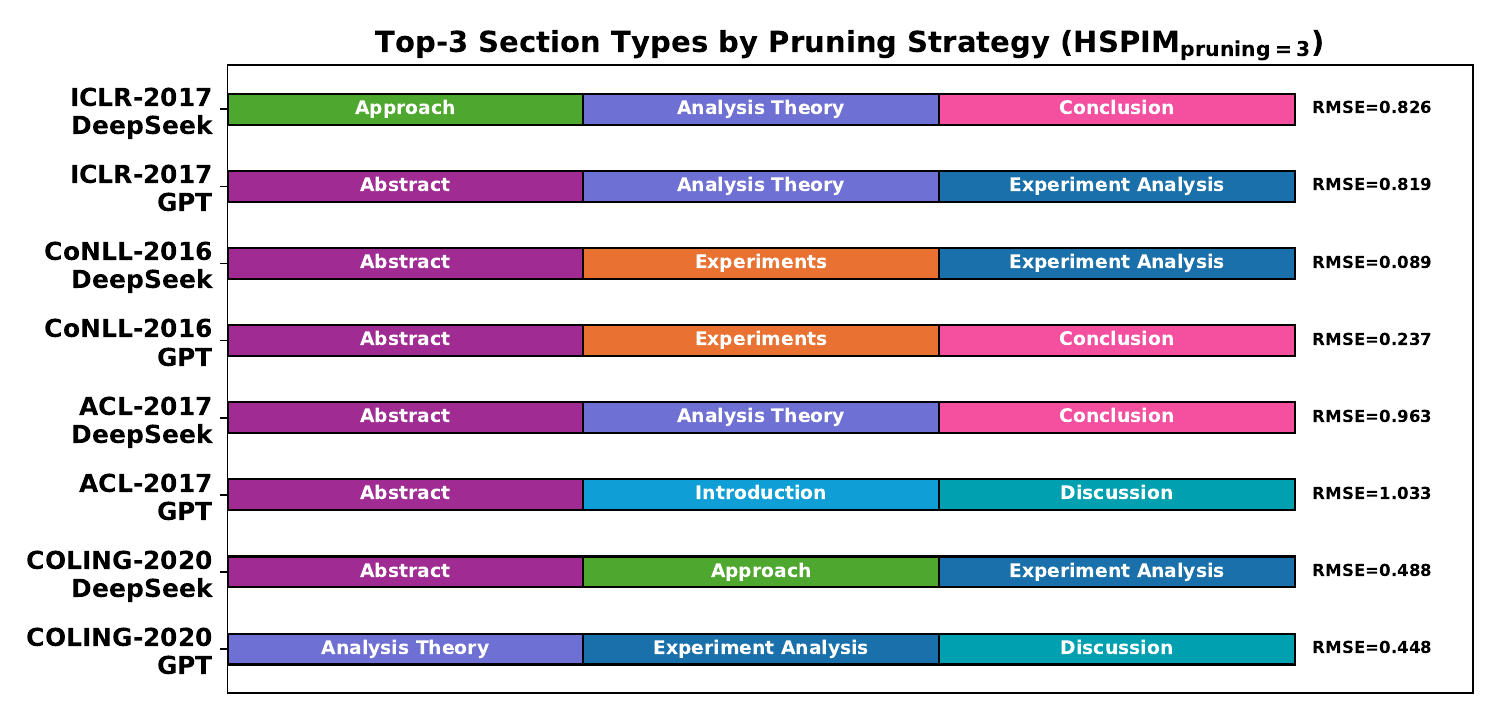}
  \caption{Top-3 section types selected by HSPIM pruning strategy (pruning=3) across different datasets and LLMs. The color bars show the combinations of section types that result in the best performance, with RMSE values shown on the right.}
  \label{pruning_plot}
\end{figure}

Additionally, we observe the following:  
(i) DeepSeek-HSPIM achieves better overall performance than GPT-HSPIM, which is reasonable because DeepSeek-V3 has stronger fundamental capabilities.
(ii) HSPIM\textsubscript{pruning=3} achieves the best results among the three optimization strategies.
Fig. \ref{pruning_plot} shows the top-3 selected section types across different datasets and LLMs. We observe that \textit{Abstract}, \textit{Analysis Theory}, \textit{Experiment Analysis}, and \textit{Conclusion} frequently appear. It indicates these sections contain important novelty explanations and method contributions for measuring innovation.
We believe section pruning helps to filter out potential noisy section types. 
In addition, it guides human reviewers to focus on specific sections that most significantly contribute to innovation analysis.
(iii) HSPIM achieves relatively lower performance on the ACL-2017 dataset due to dataset differences. 
ACL-2017 has a higher average ground-truth innovation score (3.97) on the test set than ICLR-2017 (3.57), CoNLL-2016 (3.48), and COLING-2020 (3.44). Section \ref{ssec:score_distribution} presents an analysis of the ground-truth innovation score distributions.
Supervised baselines perform better since both training and test sets in ACL-2017 have high scores. 
LLM baselines (DeepSeek-V3 and GPT-4o mini) also perform well on ACL-2017 because LLMs tend to assign higher scores to longer texts \cite{zheng2023judging}. 
This tendency matches the scoring distribution of ACL-2017. However, LLM baselines perform poorly on other datasets.

\subsubsection{Zero-Shot HSPIM vs. DL Models (Transfer from COLING-2020)}
Table \ref{tab:fair_comparison} compares HSPIM with DL models trained on COLING-2020 and evaluated on the remaining three datasets.
This setting provides a fairer comparison. HSPIM significantly outperforms DL models on ICLR-2017, CoNLL-2016, and ACL-2017. Compared to the best results of CNN, LSTM, and Transformer, HSPIM achieves average improvements of 44.0\%, 41.8\%, and 33.1\%, respectively. These results confirm the model generalization ability of the zero-shot HSPIM framework.

\begin{table}[ht]
\centering
\caption{RMSE comparison across our proposed frameworks. The bold values represent the best results.}
\label{tab:HSPIM_framework_comparison}
% \setlength{\tabcolsep}{4.5pt}
% \resizebox{\textwidth}{!}{%
\scriptsize
\setlength{\tabcolsep}{2.5pt}
\begin{tabular}{@{}c|c|c|c|c|c@{}}
\toprule
\multicolumn{1}{c}{Dataset} &  & LLMs & SSPIM & Naive HSPIM & HSPIM$_{\text{joint}}$ \\ 
\midrule
\midrule
\multirow{4}{*}{ICLR-2017}  
    & \multirow{2}{*}{train} 
        & DeepSeek & 0.9908 & 0.9366 & 0.9316 \\
    & 
        & GPT      & 0.9687 & 0.9688 & \textbf{0.9237} \\ 
    \cmidrule(l){2-6}
    & \multirow{2}{*}{test}  
        & DeepSeek & 0.8656 & 0.8531 & 0.8283 \\
    & 
        & GPT      & \textbf{0.7841} & 0.9143 & 0.8474 \\ 
\midrule
\multirow{4}{*}{CoNLL-2016} 
    & \multirow{2}{*}{train} 
        & DeepSeek & 0.6644 & 0.5710 & \textbf{0.5657} \\
    & 
        & GPT      & 0.6553 & 0.6464 & 0.6025 \\ 
    \cmidrule(l){2-6}
    & \multirow{2}{*}{test}  
        & DeepSeek & 0.2063 & 0.2184 & \textbf{0.0638} \\
    & 
        & GPT      & 0.3001 & 0.0974 & 0.1973 \\ 
\midrule
\multirow{4}{*}{ACL-2017} 
    & \multirow{2}{*}{train} 
        & DeepSeek & 1.0406 & 0.8745 & \textbf{0.8668} \\
    & 
        & GPT      & 0.9252 & 0.9906 & 0.8756 \\ 
    \cmidrule(l){2-6}
    & \multirow{2}{*}{test}  
        & DeepSeek & 1.3505 & 1.1216 & \textbf{1.0887} \\
    & 
        & GPT      & 1.1448 & 1.1801 & 1.1205 \\ 
\midrule
\multirow{4}{*}{COLING-2020} 
    & \multirow{2}{*}{train} 
        & DeepSeek & 0.6957 & 0.7077 & 0.6704 \\
    & 
        & GPT      & 0.6411 & 0.6542 & \textbf{0.4692} \\ 
    \cmidrule(l){2-6}
    & \multirow{2}{*}{test} 
        & DeepSeek & 0.5645 & 0.6339 & 0.5202 \\
    & 
        & GPT      & 0.5151 & 0.5538 & \textbf{0.5075} \\
\bottomrule
\end{tabular}
% }
% \\[3pt]
% \raggedright
% The bold values represent the best results.
\end{table}

\subsubsection{HSPIM Framework Comparison}
Table \ref{tab:HSPIM_framework_comparison} shows the performance comparison of three framework variants: section-based scientific paper innovation measurement (SSPIM), HSPIM using the initial prompt combination (naive HSPIM), and HSPIM with GA-based joint optimization strategy (HSPIM\textsubscript{joint}). We report results on both the training set and the test set. 
We observe a consistent decrease in RMSE from SSPIM to naive HSPIM, and further to HSPIM\textsubscript{joint}, which highlights the importance of each component.
Specifically, naive HSPIM outperforms SSPIM, which shows the effectiveness of \textbf{QA augmentation}. HSPIM\textsubscript{joint} improves naive HSPIM, which confirms the benefit of \textbf{multi-prompt optimization}.

\begin{table*}[ht]
\centering
\caption{Performance of GPT-HSPIM\textsuperscript{+} with different norms and optimization strategies across datasets.}
\label{tab:HSPIM_plus_results}
\scriptsize
% \renewcommand{\arraystretch}{1.3}
% \setlength{\tabcolsep}{4pt}
% \resizebox{\textwidth}{!}{%
\begin{tabular}{ll|cc|cc|cc|cc|cc}
\toprule
\multirow{2}{*}{Norm} & \multirow{2}{*}{Model} 
& \multicolumn{2}{c|}{ICLR-2017} 
& \multicolumn{2}{c|}{CoNLL-2016}
& \multicolumn{2}{c|}{ACL-2017}
& \multicolumn{2}{c|}{COLING-2020}
& \multicolumn{2}{c}{\textit{Average}} \\
& & RMSE & MAE & RMSE & MAE & RMSE & MAE & RMSE & MAE & RMSE & MAE \\
\midrule
\multirow{3}{*}{$L_1$}
& HSPIM\textsuperscript{+}\textsubscript{joint}     & 0.7695 & 0.6090 & 0.3992 & 0.3947 & 0.9461 & 0.7865 & 0.4604 & 0.3806 & 0.6438 & 0.5427 \\
& HSPIM\textsuperscript{+}\textsubscript{two-step} & 0.7904 & 0.6372 & 0.4279 & 0.4128 & 0.9386 & 0.7801 & 0.4529 & 0.3870 & 0.6525 & 0.5543 \\
& HSPIM\textsuperscript{+}\textsubscript{pruning=3}  & \textbf{0.3944} & \textbf{0.2667} & \textbf{0.2873} & \textbf{0.2870} & 0.8440 & 0.7262 & \textbf{0.2873} & \textbf{0.2870} & \textbf{0.4533} & \textbf{0.3917} \\
\midrule
\multirow{3}{*}{$L_2$}
& HSPIM\textsuperscript{+}\textsubscript{joint}     & 0.7693 & 0.6078 & 0.4038 & 0.3993 & 0.9425 & 0.7848 & 0.4621 & 0.3827 & 0.6444 & 0.5437 \\
& HSPIM\textsuperscript{+}\textsubscript{two-step} & 0.7893 & 0.6349 & 0.4318 & 0.4169 & 0.9348 & 0.7781 & 0.4533 & 0.3873 & 0.6523 & 0.5543 \\
& HSPIM\textsuperscript{+}\textsubscript{pruning=3}  & 0.3970 & 0.2714 & 0.2986 & 0.2981 & 0.8409 & 0.7246 & 0.2986 & 0.2981 & 0.4588 & 0.3980 \\
\midrule
\multirow{3}{*}{$L_\infty$}
& HSPIM\textsuperscript{+}\textsubscript{joint}     & 0.8142 & 0.5864 & 0.6336 & 0.6254 & 0.7978 & 0.7047 & 0.5833 & 0.4795 & 0.7072 & 0.5990 \\
& HSPIM\textsuperscript{+}\textsubscript{two-step} & 0.7845 & 0.5655 & 0.6473 & 0.6389 & 0.7854 & 0.6899 & 0.5330 & 0.4569 & 0.6875 & 0.5878 \\
& HSPIM\textsuperscript{+}\textsubscript{pruning=3}  & 0.5217 & 0.3667 & 0.6473 & 0.6389 & \textbf{0.7676} & \textbf{0.6786} & 0.6473 & 0.6389 & 0.6460 & 0.5808 \\
\bottomrule
\end{tabular}
% }
\end{table*}

\subsubsection{HSPIM$^+$ Results}
\label{ssec:HSPIM$^+$}
Following Definition \ref{def:HSPIM$^+$}, HSPIM$^+$ measures innovation as a norm-based combination of novelty, contribution, and feasibility. As shown in Table \ref{tab:HSPIM_plus_results}, the choice of aggregation norm ($L_1$, $L_2$, $L_\infty$) has a notable effect on performance. 

The results in Table~\ref{tab:HSPIM_plus_results} show that the $L_1$ and $L_2$ norms yield similar performance across all datasets and optimization strategies. 
Notably, the $L_\infty$ norm achieves the lowest errors on the ACL-2017 dataset, outperforming both $L_1$ and $L_2$.
For all norm choices, the pruning strategy achieves the lowest RMSE and MAE, indicating its effectiveness.

Compared to the original HSPIM framework (Table~\ref{tab:unfair_comparison}), HSPIM$^+$ achieves a lower best average RMSE and MAE (0.4533 and 0.3917) than HSPIM (0.5915 and 0.4681).
This improvement is particularly clear on ACL-2017, where HSPIM$^+$ outperforms zero-shot LLM baselines and matches the performance of supervised models. 
These results demonstrate that modeling innovation as a composite of novelty, contribution, and feasibility through norm-based aggregation is both theoretically justified and empirically validated.

% \begin{table}[ht]
% \centering
% \caption{
% Performance comparison on a subset of the PMC Article Dataset and the arXiv physics.ins-det corpus. 
% \textbf{Bold} highlights the best results and \underline{underline} highlights the second best results.
% % The bold values indicate the best results.
% }
% \label{tab:bmc_arxiv_dev_test}
% \scriptsize
% \setlength{\tabcolsep}{0.5pt}
% % \renewcommand{\arraystretch}{1.4}
% \begin{tabular}{l|cc|cc}
% \toprule
% Model & RMSE (PMC) & MAE (PMC) & RMSE (arXiv) & MAE (arXiv) \\
% \midrule
% \multicolumn{1}{l}{} & \multicolumn{4}{c}{\textit{Supervised Models}} \\
% % \multicolumn{1}{l}{} & \multicolumn{2}{c}{\textit{Supervised Models}} \\
% CNN & 1.5481 & 1.3022 & 1.4870 & 1.2244 \\
% RNN & 1.0187 & 0.7878 & 1.1123 & 0.8823 \\
% Transformer & \underline{0.8884} & \underline{0.6814} & \textbf{0.9380} & \textbf{0.7540} \\
% \midrule
% \multicolumn{1}{l}{} & \multicolumn{4}{c}{\textit{Zero-Shot LLMs}} \\
% % GPT-4o mini & 3.2476 & 3.0000 & --- & --- \\
% GPT-4o mini & 3.1006 & 2.9105 & 1.5711 & 1.3542 \\
% DeepSeek-V3 & 3.2957 & 3.1779 & 1.2333 & 1.0221 \\
% \midrule
% GPT-HSPIM & \textbf{0.8824} & 0.6853 & 1.0371 & 0.8256 \\
% DeepSeek-HSPIM & 0.8956 & \textbf{0.6755} & \underline{0.9424} & \underline{0.7554} \\
% \bottomrule
% \end{tabular}
% \end{table}

\begin{table}[ht]
\centering
\caption{
Performance comparison on a subset of the PMC Article Dataset and the arXiv physics.ins-det corpus. 
\textbf{Bold} highlights the best results and \underline{underline} highlights the second best results.
}
\label{tab:bmc_arxiv_dev_test}
\scriptsize
\begin{tabular}{l|cc|cc}
\toprule
Datasets & \multicolumn{2}{c|}{PMC} & \multicolumn{2}{c}{arXiv} \\ % new top header row
\cmidrule(lr){2-3}\cmidrule(lr){4-5}
Metrics & RMSE & MAE & RMSE & MAE \\ % metric row (no dataset names here)
\midrule
\multicolumn{1}{l}{} & \multicolumn{4}{c}{\textit{Supervised Models}} \\
CNN & 1.5481 & 1.3022 & 1.4870 & 1.2244 \\
RNN & 1.0187 & 0.7878 & 1.1123 & 0.8823 \\
Transformer & \underline{0.8884} & \underline{0.6814} & \textbf{0.9380} & \textbf{0.7540} \\
\midrule
\multicolumn{1}{l}{} & \multicolumn{4}{c}{\textit{Zero-Shot LLMs}} \\
GPT-4o mini & 3.1006 & 2.9105 & 1.5711 & 1.3542 \\
DeepSeek-V3 & 3.2957 & 3.1779 & 1.2333 & 1.0221 \\
\midrule
GPT-HSPIM & \textbf{0.8824} & 0.6853 & 1.0371 & 0.8256 \\
DeepSeek-HSPIM & 0.8956 & \textbf{0.6755} & \underline{0.9424} & \underline{0.7554} \\
\bottomrule
\end{tabular}
\end{table}

\subsubsection{Domain Generalization Experiments}
\label{ssec:domain_generalization}
To evaluate the domain generalization performance of the HSPIM framework, we conduct experiments on subsets of the PubMed Central (PMC) Article Datasets and the arXiv physics.ins-det (Instrumentation and Detectors) collection. 
Since both datasets lack peer review scores, following previous studies, we use citation counts as a proxy for innovation scores \cite{amplayo2018network}. 
We then map citation counts to ground-truth innovation scores via a nonlinear transformation that approximates a normal distribution, so that the resulting score distribution closely fits the ground-truth innovation score distribution observed in Section \ref{ssec:score_distribution}. 
For PMC, we select 95 biochemical articles published in 2004 from the journal \textit{Bioinformatics} to ensure uniformity in publication venue, year, and domain. 
For arXiv, we use the 2018 corpus of \textit{physics.ins-det} (Instrumentation and Detectors), comprising 276 papers on detector design, readout and DAQ, performance evaluation, simulation, and AI-assisted control.

As shown in Table \ref{tab:bmc_arxiv_dev_test}, HSPIM achieves the best RMSE and MAE on PMC and the second-best scores on arXiv, and it outperforms zero-shot LLM baselines on both datasets. These results indicate strong cross-domain performance and support the general applicability of HSPIM.

\begin{table}[ht]
\centering
\caption{Semantic textual similarity between LLM-generated scoring reasons and actual peer review comments. The bold values represent the best results.}
\label{tab:semantic_similarity}
\setlength{\tabcolsep}{2.5pt}
\scriptsize
\begin{tabular}{lccc|ccc}
\toprule
\multirow{2}{*}{Dataset} 
  & \multicolumn{3}{c|}{Cosine Similarity} 
  & \multicolumn{3}{c}{BERTScore} \\
\cmidrule(lr){2-4}\cmidrule(lr){5-7}
 & SEA-E & SSPIM & HSPIM$_{\text{joint}}$ & SEA-E & SSPIM & HSPIM$_{\text{joint}}$ \\
\midrule
ICLR-2017   & 0.6532 & 0.7761 & \textbf{0.8551} & 0.8402 & 0.8688 & \textbf{0.8758} \\
CoNLL-2016  & 0.6712 & 0.7938 & \textbf{0.8529} & 0.8284 & 0.8546 & \textbf{0.8627} \\
ACL-2017    & 0.6516 & 0.8131 & \textbf{0.8602} & 0.8270 & 0.8532 & \textbf{0.8600} \\
COLING-2020 & 0.7001 & 0.7847 & \textbf{0.8342} & 0.8401 & 0.8553 & \textbf{0.8581} \\
\bottomrule
\end{tabular}
% \\[3pt]
% \raggedright
% The bold values represent the best results.
\end{table}

\subsection{Semantic Textual Similarity}
Table \ref{tab:semantic_similarity} reports the semantic similarity between scoring rationales generated by LLMs and actual peer review comments. The experiment compares SEA-E, SSPIM, and HSPIM in terms of how closely their generated evaluation texts align with real reviews. 
In SSPIM and HSPIM, LLMs are prompted to output not only the \textit{``novelty\_score''} and \textit{``confidence\_score''} but also a \textit{``reason''} to explain the novelty score. To measure semantic similarity, we use cosine similarity and BERTScore \cite{zhangbertscore}. Cosine similarity is computed with sentence embeddings from the gte-large-en-v1.5 model \cite{zhang2024mgte}. BERTScore uses the longformer-base-4096 \cite{beltagy2020longformer} model to support long inputs.
HSPIM\textsubscript{joint} achieves higher scores than SEA-E and SSPIM, with an average Cosine Similarity improvement of 27.3\% and 7.4\%, and a BERTScore improvement of 3.6\% and 0.7\%, respectively.
Compared to SEA-E and SSPIM, HSPIM generates scoring rationales that are semantically closer to real peer review comments. Therefore, our HSPIM framework can not only assess innovation quantitatively but also explain it qualitatively.

\begin{table}[ht]
\centering
\caption{RMSE comparison of different LLM-based HSPIM\textsubscript{joint}. The bold values represent the best results.}
\label{tab:open_source_llm_rmse}
\scriptsize
\begin{tabular}{lccc|cc}
\toprule
Dataset & Llama & Mistral & Qwen & DeepSeek & GPT \\
\midrule
ICLR-2017    & 0.8128 & \textbf{0.7606} & 0.7968 & 0.8283 & 0.8474 \\
CoNLL-2016   & 0.2036 & 0.3309 & 0.3800 & \textbf{0.0638} & 0.1973 \\
ACL-2017     & 1.0552 & \textbf{0.9224} & 0.9559 & 1.0887 & 1.1205 \\
COLING-2020  & 0.5713 & \textbf{0.4946} & 0.5038 & 0.5202 & 0.5075 \\
\bottomrule
\end{tabular}
% \\[3pt]
% \raggedright
% The bold values represent the best results.
\end{table}

\subsection{LLM-based HSPIM Comparison}
\label{ssec:llm_hspim_comparison}
Table \ref{tab:open_source_llm_rmse} compares the RMSE results of HSPIM\textsubscript{joint} using different LLMs. In addition to DeepSeek-V3 and GPT-4o Mini, we also deploy three small open-source LLMs: LLaMA-3-8B \cite{touvron2023llama}, Mistral-7B \cite{jiang2023mistral}, and Qwen2.5-7B \cite{yang2024qwen2}.
We see that HSPIM performs well on small-scale LLMs, showing its potential to support innovation assessment with lightweight models. Notably, Mistral-7B outperforms other LLMs on three out of four datasets.

\begin{table}[h]
\centering
\caption{RMSE comparison of SSPIM and HSPIM$_{\text{joint}}$ with and without confidence scores on test sets. The bold values represent the best results.}
\label{tab:ablation_confidence_score}
\scriptsize
\begin{tabular}{lcccc}
\toprule
\multirow{2}{*}{Dataset} 
  & \multicolumn{2}{c}{SSPIM} 
  & \multicolumn{2}{c}{HSPIM$_{\text{joint}}$} \\
\cmidrule(lr){2-3}\cmidrule(lr){4-5}
  & w/o conf. & w/ conf. & w/o conf. & w/ conf. \\
\midrule
ICLR-2017   & 0.8657 & \textbf{0.8656} & \textbf{0.8262} & 0.8283 \\
CoNLL-2016  & 0.2297 & \textbf{0.2063} & 0.0660 & \textbf{0.0638} \\
ACL-2017    & 1.3730 & \textbf{1.3505} & 1.1056 & \textbf{1.0887} \\
COLING-2020 & 0.5730  & \textbf{0.5645} & 0.5273 & \textbf{0.5202} \\
\bottomrule
\end{tabular}
% \\[3pt]
% \raggedright
% The bold values represent the best results.
\end{table}

\begin{table}[h]
\centering
\caption{Effectiveness of two-layer question structure on HSPIM framework. The bold values represent the best results.}
\label{tab:ablation_prompt_design}
\setlength{\tabcolsep}{6pt}
\scriptsize
\begin{tabular}{llcc}
\toprule
Dataset & Question Prompt Type & RMSE & MAE \\
\midrule
\multirow{3}{*}{ICLR-2017} 
    & Specific           & 0.9034 & 0.7881 \\
    & Common             & 0.8642 & 0.7217 \\
    & Specific+Common    & \textbf{0.8283} & \textbf{0.6966} \\
\midrule
\multirow{3}{*}{CoNLL-2016} 
    & Specific           & 0.0793 & 0.0614 \\
    & Common             & 0.1735 & 0.1389 \\
    & Specific+Common    & \textbf{0.0638} & \textbf{0.0560} \\
\midrule
\multirow{3}{*}{ACL-2017} 
    & Specific           & 1.2948 & 1.0878 \\
    & Common             & 1.1271 & 0.8975 \\
    & Specific+Common    & \textbf{1.0887} & \textbf{0.8572} \\
\midrule
\multirow{3}{*}{COLING-2020} 
    & Specific           & 0.5899 & 0.4618 \\
    & Common             & 0.5359 & 0.4201 \\
    & Specific+Common    & \textbf{0.5202} & \textbf{0.3933} \\
\bottomrule
\end{tabular}
% \\[3pt]
% \raggedright
% The bold values represent the best results.
\end{table}

\section{Experimental Analysis}
% \minew{In this section, we conduct several experimental analyses to comprehensively demonstrate the effectiveness and robustness of the HSPIM framework and to present diverse results.}

\subsection{Ablation Study}

In this section, we conduct ablation study on confidence score and two-layer question structure.

\subsubsection{Ablation Study on Confidence Score}
Table \ref{tab:ablation_confidence_score} compares the performance with (w/) and without (w/o) using confidence scores. 
The table reports RMSE results for SSPIM and HSPIM\textsubscript{joint} on DeepSeek-V3. Using confidence scores to do weighted innovation scoring helps to reduce RMSE in both SSPIM and HSPIM frameworks. This result shows that adding confidence scores is helpful to connect local novelty evaluations to full-paper innovation.

\subsubsection{Ablation Study on Two-Layer Question Structure}
Table \ref{tab:ablation_prompt_design} presents the results of the ablation study on two-layer question structure under DeepSeek-HSPIM.
We can see that combining specific and common questions helps to reduce RMSE by 13.90\% compared to using only specific questions, and by 18.43\% compared to using only common questions. 
These results confirm the effectiveness of the two-layer question structure.

\begin{figure*}[!ht]
    \centering
    \subfloat[\scriptsize ICLR-2017]{
        \includegraphics[width=0.4\linewidth]{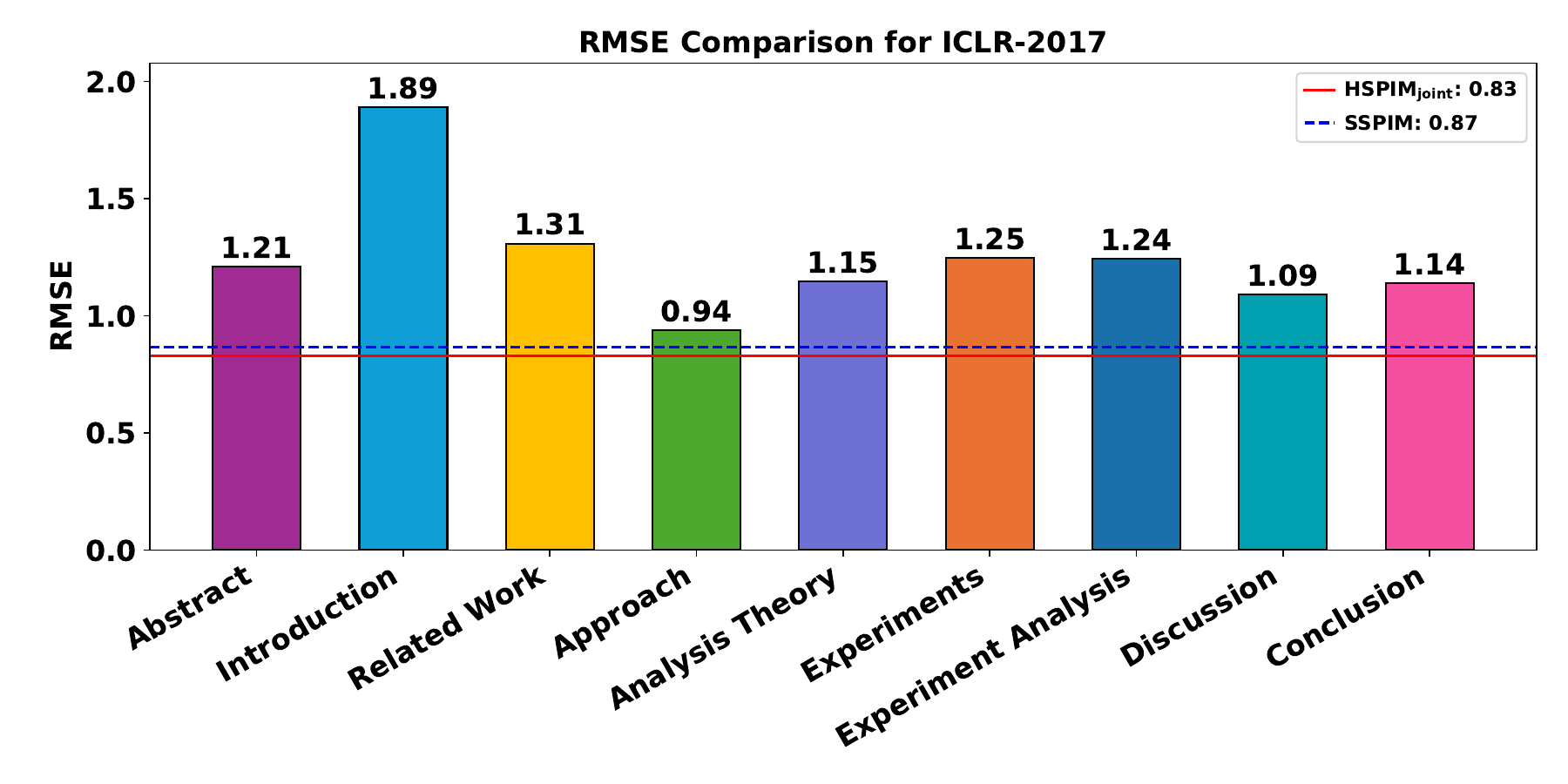}
        \label{fig:iclr2017}
    }
    \subfloat[\scriptsize CoNLL-2016]{
        \includegraphics[width=0.4\linewidth]{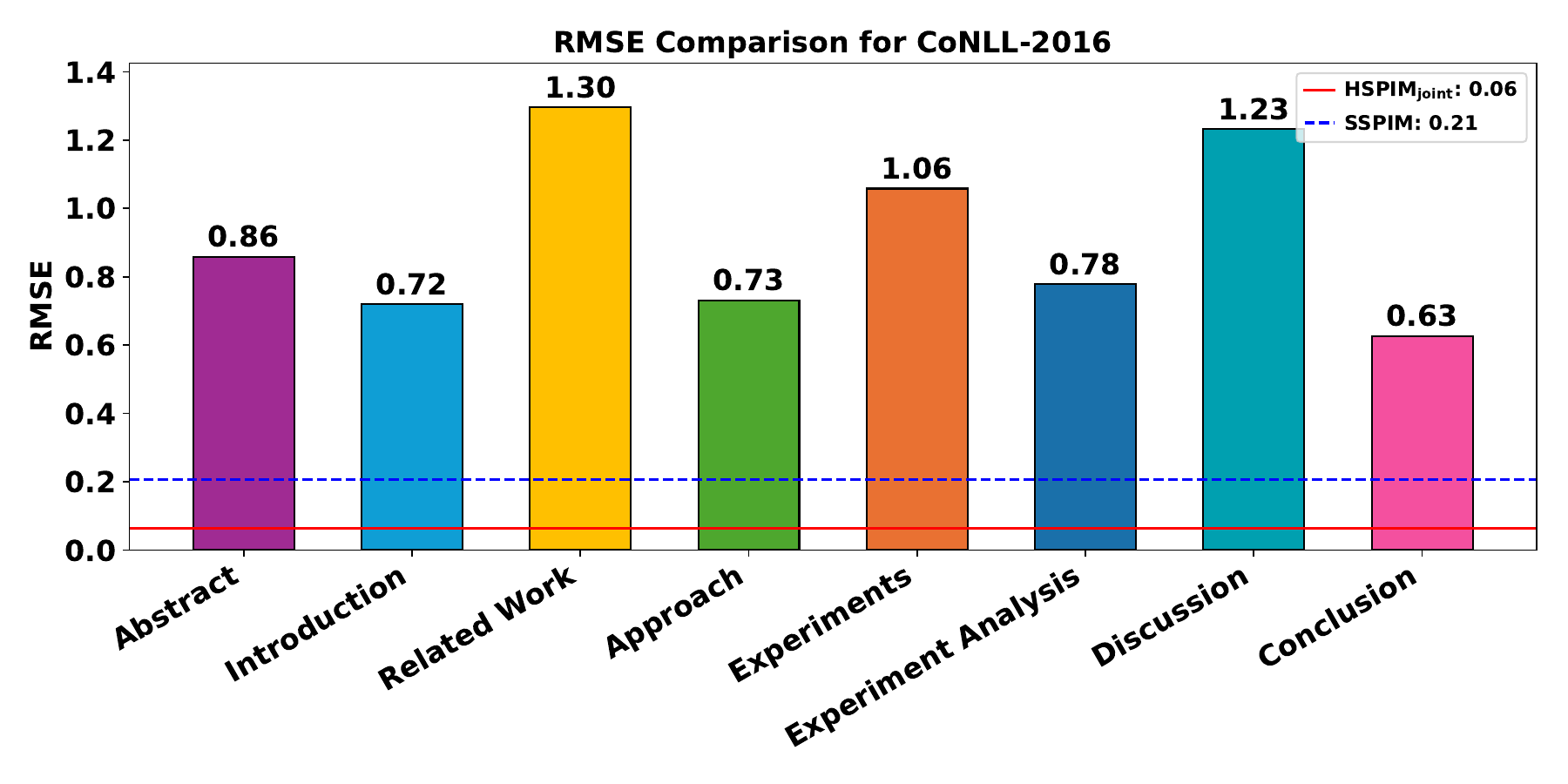}
        \label{fig:conll2016}
    }
    
    \subfloat[\scriptsize ACL-2017]{
        \includegraphics[width=0.4\linewidth]{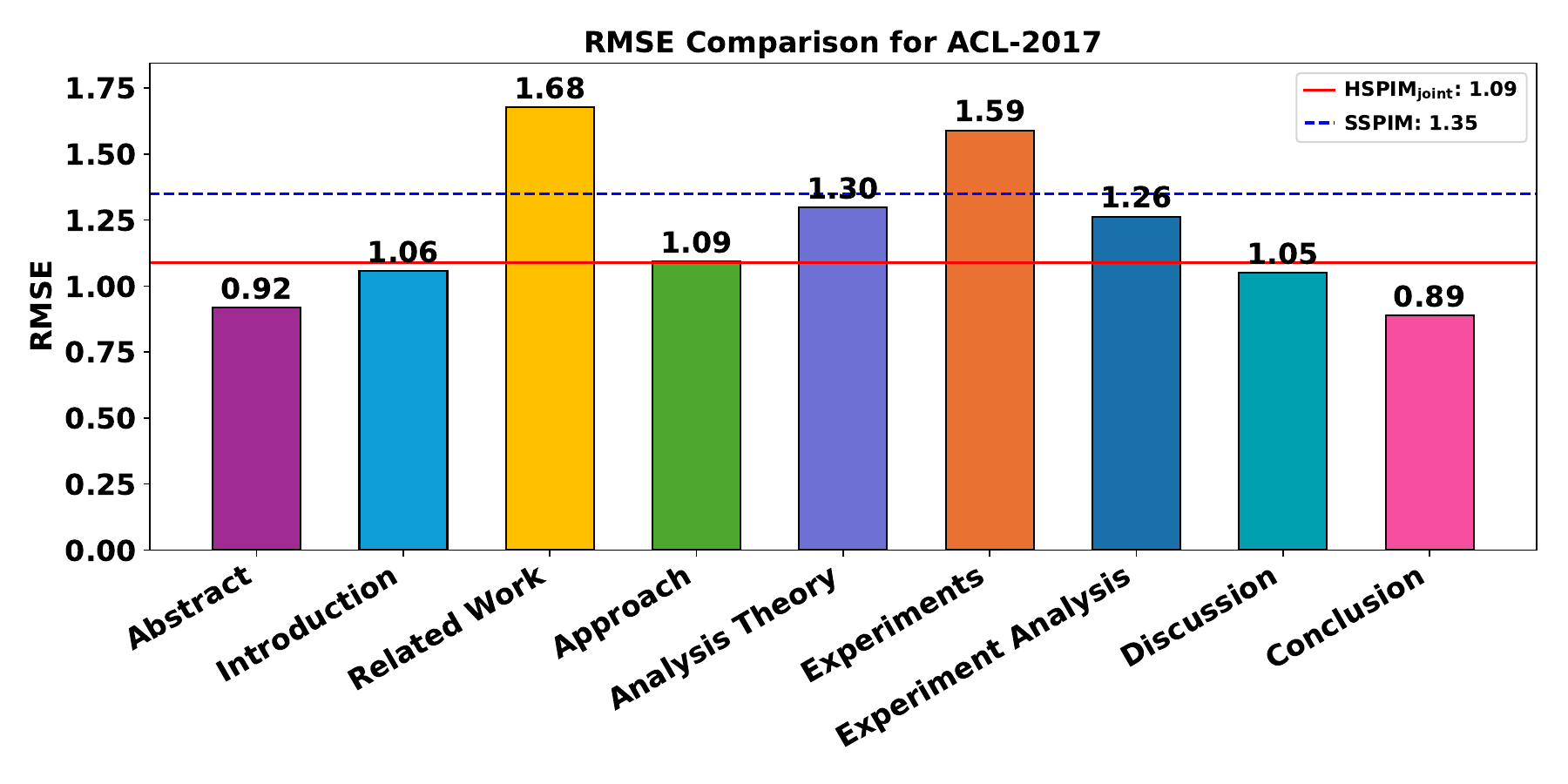}
        \label{fig:acl2017}
    }
    \subfloat[\scriptsize COLING-2020]{
        \includegraphics[width=0.4\linewidth]{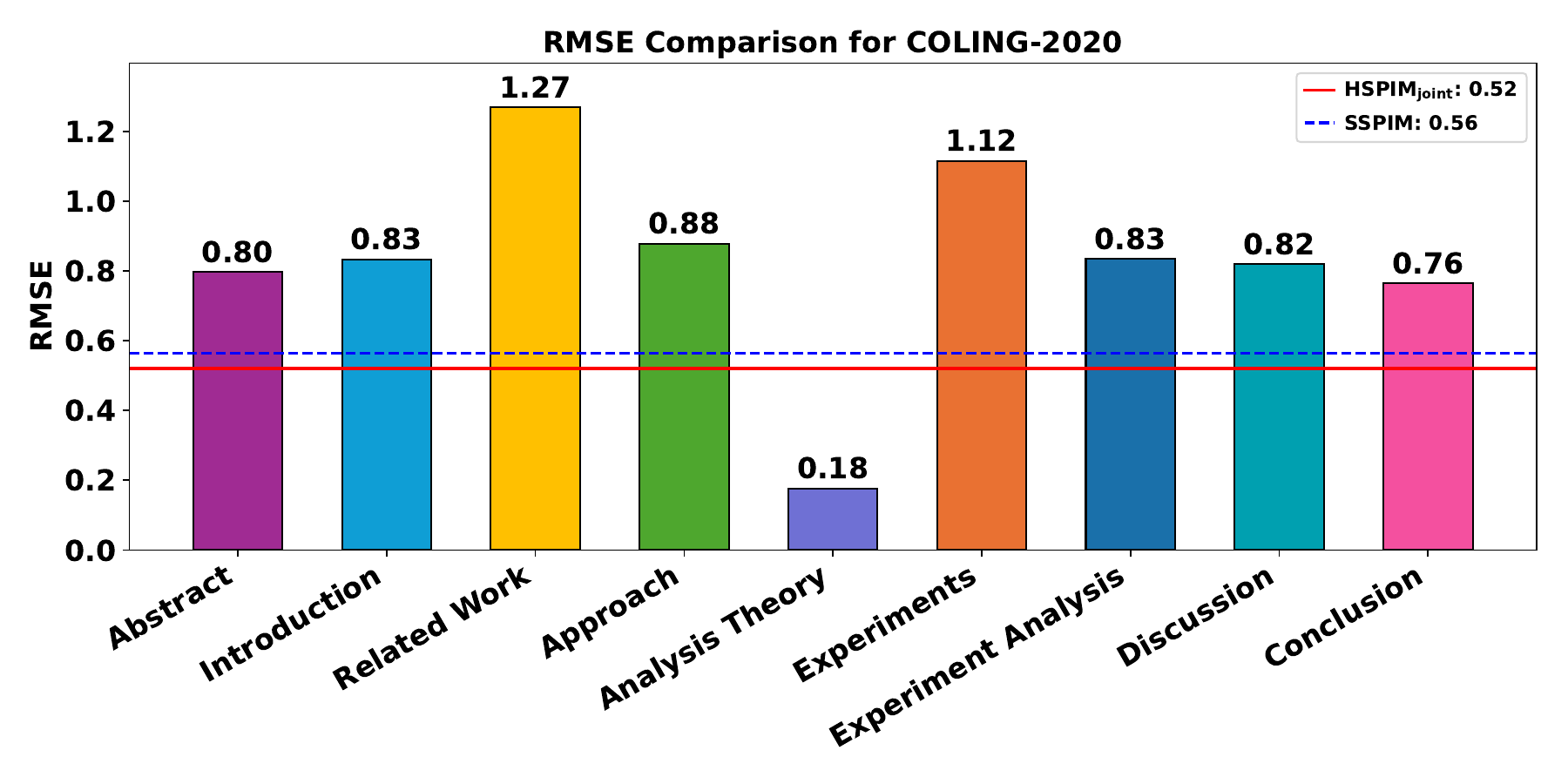}
        \label{fig:coling2020}
    }
    \caption{
    Innovation scoring RMSE comparison of using individual sections, SSPIM, and HSPIM$_{\text{joint}}$, based on DeepSeek-V3. Each bar represents the RMSE for the average innovation score of a section type.
    The red solid line and the blue dashed line represent the RMSE results of HSPIM\textsubscript{joint} and SSPIM, respectively.
    }
    \label{fig:rmse_comparison}
\end{figure*}

\begin{figure*}[!ht]
    \centering
    \subfloat[\scriptsize ACL-2017]{
        \includegraphics[width=0.35\linewidth]{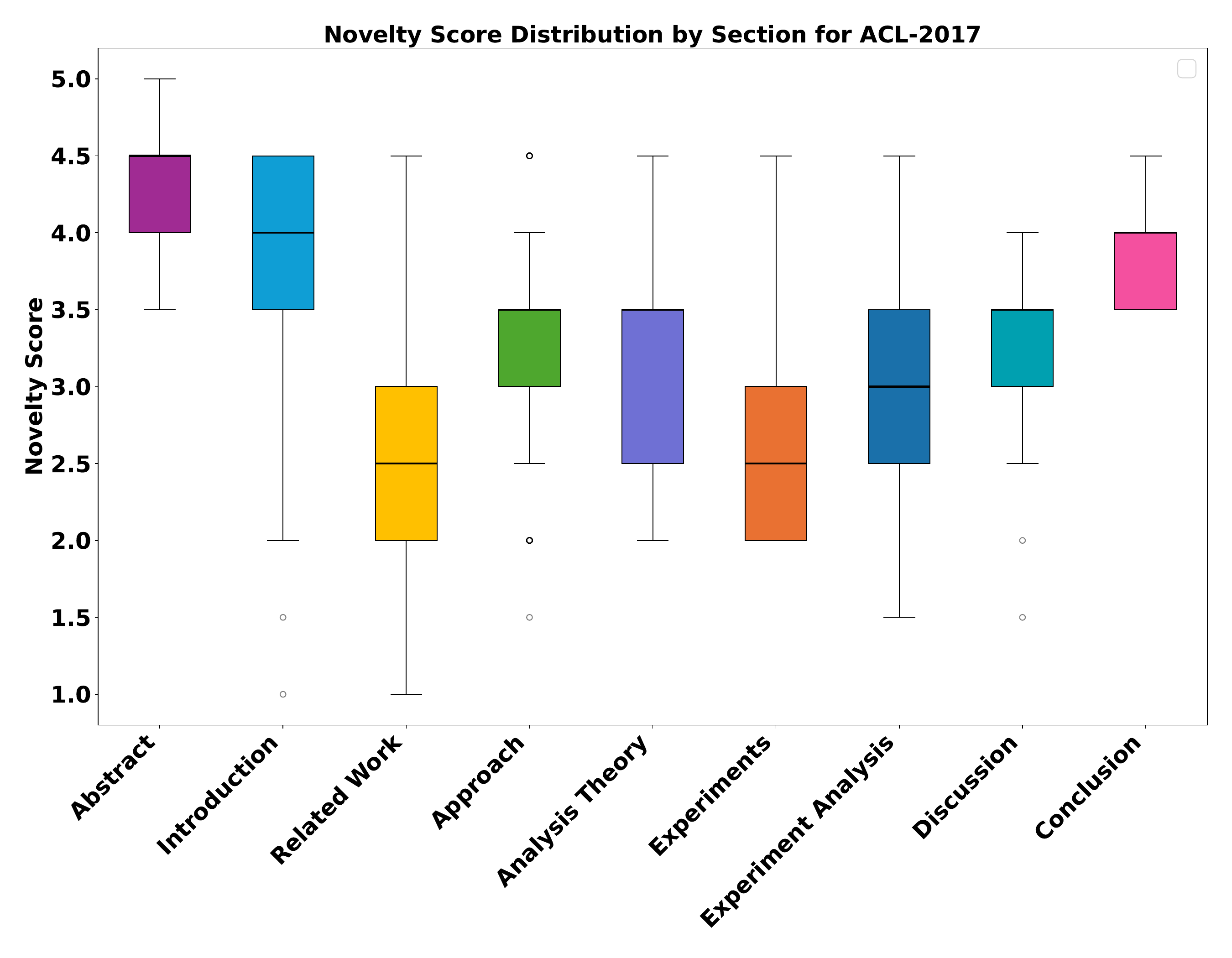}
        \label{fig:acl_2017_novelty_scores}
    }
    \subfloat[\scriptsize CoNLL-2016]{
        \includegraphics[width=0.35\linewidth]{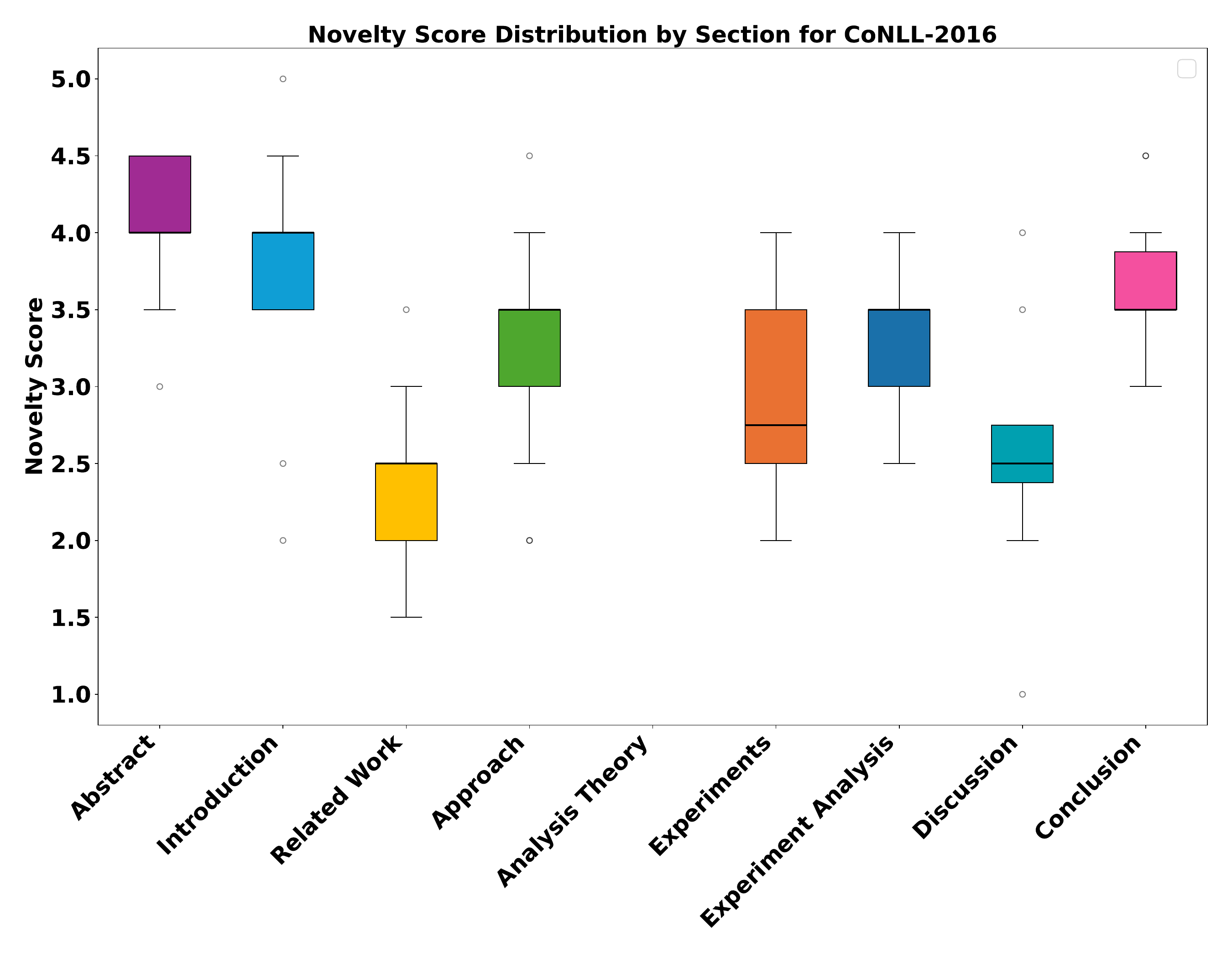}
        \label{fig:conll_2016_novelty_scores}
    }
    
    \subfloat[\scriptsize ICLR-2017]{
        \includegraphics[width=0.35\linewidth]{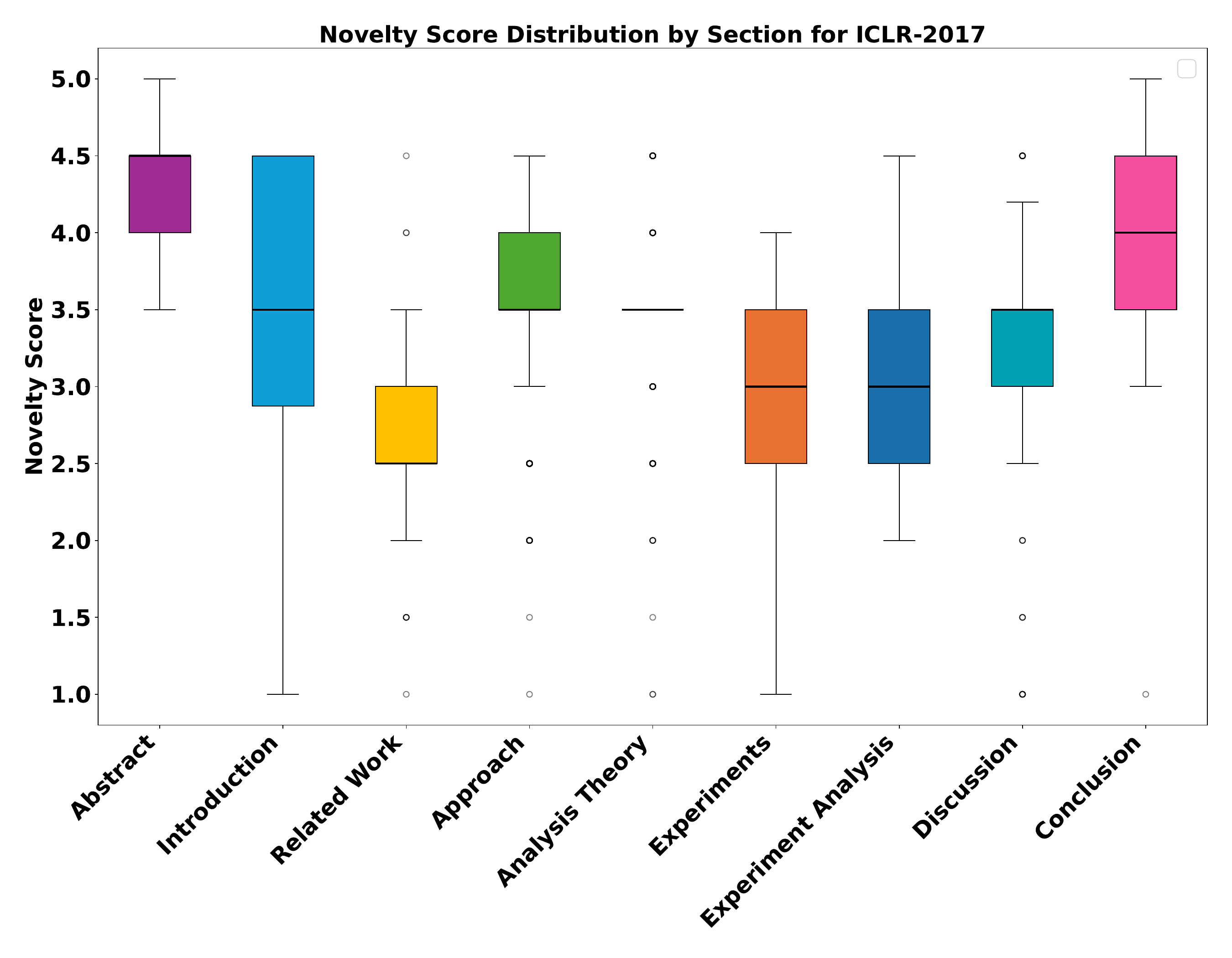}
        \label{fig:iclr_2017_novelty_scores}
    }
    \subfloat[\scriptsize COLING-2020]{
        \includegraphics[width=0.35\linewidth]{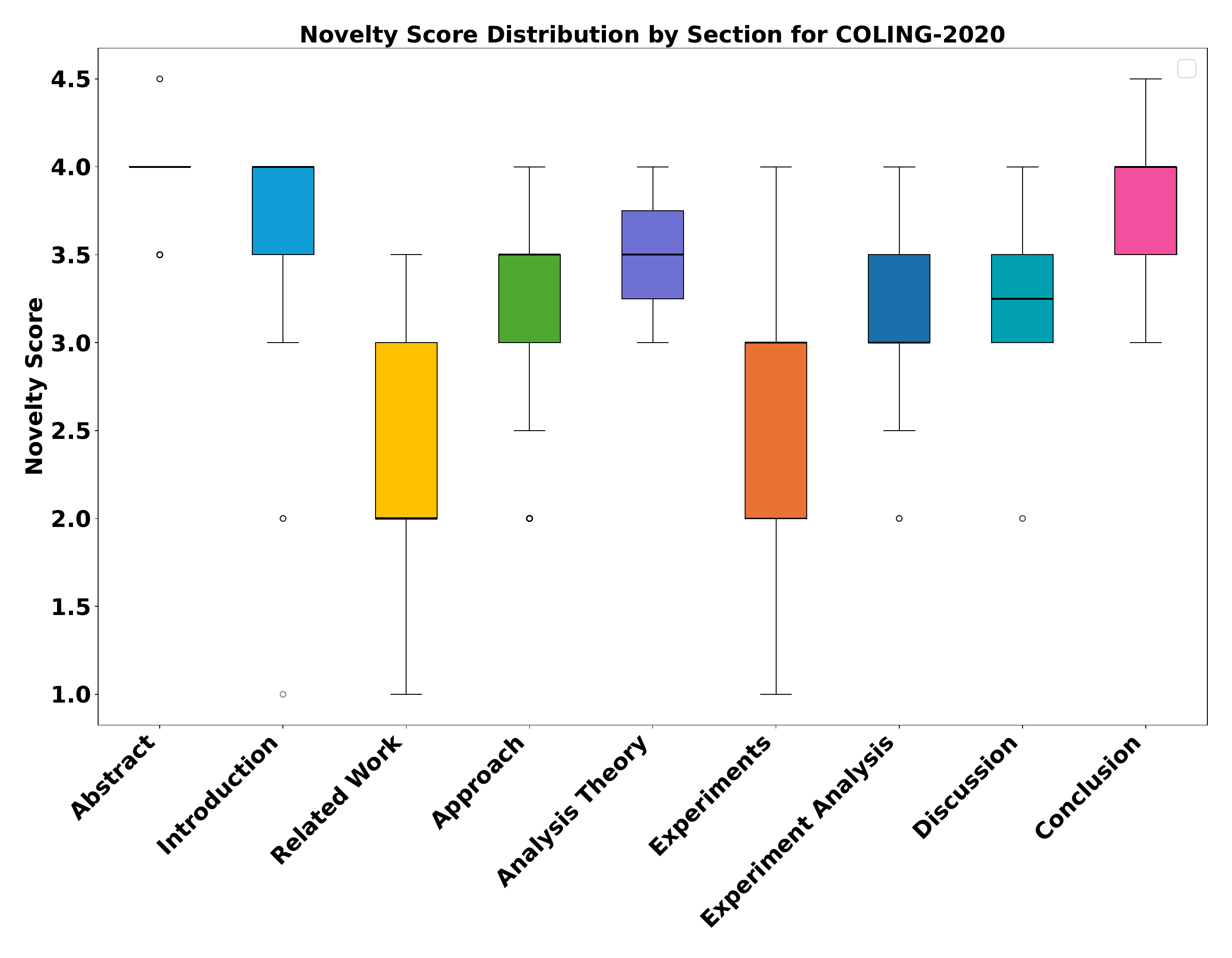}
        \label{fig:coling_2020_novelty_scores}
    }
    \caption{Novelty scoring tendencies across different section types on DeepSeek-HSPIM.}
    \label{fig:novelty_scoring_comparison}
\end{figure*}

\subsection{Analysis of Section-Based Scoring}

In this section, we conduct analyses of section-based scoring.

\subsubsection{Effect of Section Selection on Innovation Scoring}
Fig. \ref{fig:rmse_comparison} compares the RMSE of using individual sections (scored by LLM one-step prompting), SSPIM, and HSPIM\textsubscript{joint} across four datasets. We observe that HSPIM\textsubscript{joint} consistently achieves the lowest RMSE on ICLR-2017, CoNLL-2016, and COLING-2020, and also outperforms five out of nine single-section baselines on ACL-2017. The results highlight the effectiveness of section-based weighted innovation scoring (from individual sections to SSPIM and HSPIM).
In addition, HSPIM provides more stable RMSE performance than using single sections.

\subsubsection{Scoring Bias Across Sections}
The RMSE differences across section types shown in Fig. \ref{fig:rmse_comparison} are essentially driven by the inconsistent scoring tendencies of LLMs across sections.
Fig. \ref{fig:novelty_scoring_comparison} presents box plots that reveal systematic scoring biases across different section types. LLMs tend to assign higher novelty scores to sections like \textit{Abstract}, \textit{Introduction}, and \textit{Conclusion}, while assigning lower scores to sections such as \textit{Related Work} and \textit{Discussion}.
We observe that \textit{Abstract} and \textit{Conclusion} receive higher scores with lower variance, while \textit{Related Work} and \textit{Experiments} show lower scores and greater spread. 
These findings underscore the limitations of single section analysis, which may introduce bias and instability.
To address this, HSPIM uses weighted scoring across sections to better reflect full paper innovation.

\begin{table*}[ht]
\centering
\caption{RMSE comparison of representative section combinations on HSPIM. The bold values represent the best results.}
\label{tab:section_combination}
% \setlength{\tabcolsep}{4pt}
% \scriptsize
\resizebox{\textwidth}{!}{%
\begin{tabular}{l|l|c|c|c|c}
\toprule
Category & Section Combination & ICLR‑17 & CoNLL‑16 & ACL‑17 & COLING‑20 \\
\midrule
\multirow{2}{*}{Background Context}
 & \textit{Introduction + Related Work}                       & 0.8584 & 0.3206 & 1.0802 & 0.5074 \\
 & \textit{Related Work + Approach}                           & 0.8842 & 0.3763 & 1.0801 & 0.5259 \\
\midrule
\multirow{2}{*}{Method \& Evidence}
 & \textit{Approach + Experiments}                            & 0.8969 & 0.2059 & 1.0395 & \textbf{0.4750} \\
 & \textit{Approach + Analysis Theory}                        & 0.8563 & 0.3915 & 1.0766 & 0.5255 \\
\midrule
\multirow{2}{*}{Findings \& Discussion}
 & \textit{Experiments + Discussion}                          & 0.8395 & 0.2148 & 1.0701 & 0.4948 \\
 & \textit{Experiments + Experiment Analysis}                 & 0.8636 & 0.2141 & 1.0603 & 0.4878 \\
\midrule
\multirow{2}{*}{Summative Insight}
 & \textit{Introduction + Discussion + Conclusion}            & \textbf{0.8200} & 0.3583 & 1.0866 & 0.5210 \\
 & \textit{Introduction + Conclusion}                         & 0.8256 & 0.3997 & 1.0804 & 0.5195 \\
\midrule
\multirow{2}{*}{Comprehensive Coverage}
 & \textit{Introduction + Related Work + Approach + Experiments} & 0.9334 & 0.2165 & \textbf{1.0035} & 0.4836 \\
 & \textit{Experiments + Analysis + Discussion + Conclusion}      & 0.8711 & \textbf{0.1738} & 1.0666 & 0.4881 \\
\midrule
\multicolumn{2}{c|}{All Sections, Independent} & 0.8474 & 0.1973 & 1.1205 & 0.5075 \\
\bottomrule
\end{tabular}
}
% \\[3pt]
% \raggedright
% The bold values represent the best results.
\end{table*}

\subsubsection{Modeling Section Combination}
In the HSPIM framework, all section chunks are processed independently through LLM prompting, including question answering and scoring. In practice, sections may contain overlapping or complementary information. To test whether merging sections helps improve performance, we conduct a section combination experiment.
Specifically, we concatenate the texts of two or more sections into a single input chunk and treat it as a new combined section for HSPIM, while keeping all other sections unchanged.

Table \ref{tab:section_combination} shows the RMSE results for different section combining strategies using GPT-HSPIM\textsubscript{joint}. 
We group these strategies into five categories based on the functional relationship between the merged sections. 
For example, the combination \textit{Approach + Experiments} merges the method and experiment content to capture their complementary information.
\textit{All Sections, Independent} denotes the original setup where no sections are merged.
We observe that some section combinations outperform the original setting on certain datasets.
Compared to the original setup, HSPIM with combined sections shows performance gains in several cases, such as \textit{Approach + Experiments}, \textit{Introduction + Discussion + Conclusion}, \textit{Introduction + Related Work + Approach + Experiments}, and \textit{Experiments + Analysis + Discussion + Conclusion}.
Hence, combining sections is a potential improvement path for the HSPIM framework.

\subsection{Sensitivity Analysis}
In this section, we conduct sensitivity analyses of HSPIM with respect to LLM temperature, prompt instructions for critical scoring, and self-praising expressions in scientific papers.

\begin{table}[ht]
\centering
\caption{Performance of GPT-HSPIM\textsubscript{joint} under different LLM temperature settings.}
\label{tab:llm_temp}
\scriptsize
\setlength{\tabcolsep}{0.5pt}
\begin{tabular}{lcccccccccc}
\toprule
\multirow{3}{*}{Dataset}
  & \multicolumn{8}{c}{LLM Temperature} \\
\cmidrule(lr){2-9}
  & \multicolumn{2}{c}{0} 
  & \multicolumn{2}{c}{0.3}
  & \multicolumn{2}{c}{0.7}
  & \multicolumn{2}{c}{1.0} \\
\cmidrule(lr){2-3} \cmidrule(lr){4-5} \cmidrule(lr){6-7} \cmidrule(lr){8-9}
& RMSE & MAE & RMSE & MAE & RMSE & MAE & RMSE & MAE \\
\midrule
ICLR-2017
    & 0.8474 & 0.7160
    & 0.8499 & 0.6399
    & \textbf{0.8399} & \textbf{0.6359}
    & 0.8532 & 0.6463 \\
CoNLL-2016  
    & 0.1973 & 0.1736
    & \textbf{0.0663} & \textbf{0.0593}
    & 0.1321 & 0.1197
    & 0.1587 & 0.1831 \\
ACL-2017    
    & \textbf{1.1205} & \textbf{0.8897}
    & 1.1489 & 0.9071
    & 1.1432 & 0.9629
    & 1.1395 & 0.9267 \\
COLING-2020 
    & 0.5075 & 0.4129
    & 0.4872 & 0.3901
    & \textbf{0.4844} & \textbf{0.3875}
    & 0.4593 & 0.3666 \\
% BMC\_2004
%     & 0.9345 & 0.7772
%     & 0.9345 & 0.7772
%     & 0.8902 & 0.7076
%     & 0.8883 & 0.7049 \\
\bottomrule
\end{tabular}
\end{table}

\subsubsection{Sensitivity Analysis on LLM Temperature}
We analyze the effect of the LLM temperature parameter on the performance of GPT-HSPIM\textsubscript{joint}. 
We analyze the effect of the LLM temperature parameter on the performance of GPT-HSPIM\textsubscript{joint}. In all other experiments, we set the temperature to $0$ to reduce randomness. To increase diversity, we set the temperature to higher values such as $1$.

As shown in Table \ref{tab:llm_temp}, varying the temperature from $0$ to $1.0$ has only a limited impact on RMSE and MAE across all datasets. The results show that increasing the temperature does not significantly change model predictions or reduce conservative scoring. HSPIM framework remains robust to this parameter.

\begin{table}[ht!]
\centering
\caption{Effect of adding the critical scoring instruction in the scoring prompt on GPT-HSPIM\textsubscript{joint}.}
\label{tab:prompt_sensitivity}
\scriptsize
\setlength{\tabcolsep}{2.5pt}
\begin{tabular}{l l cc}
\toprule
Dataset & Setting & RMSE & MAE \\
\midrule
\multirow{3}{*}{ICLR-2017}
    & Original       & 0.8474 & 0.7160 \\
    & +Critical Scoring & 0.7939 & 0.6752 \\
    & Change (\%)    & $-$6.31\% & $-$5.70\% \\
\cmidrule(lr){1-4}
\multirow{3}{*}{CoNLL-2016}
    & Original       & 0.1973 & 0.1736 \\
    & +Critical Scoring & 0.1557 & 0.1464 \\
    & Change (\%)    & $-$21.08\% & $-$15.67\% \\
\cmidrule(lr){1-4}
\multirow{3}{*}{ACL-2017}
    & Original       & 1.1205 & 0.8897 \\
    & +Critical Scoring & 1.1614 & 0.9436 \\
    & Change (\%)    & +3.65\% & +6.06\% \\
\cmidrule(lr){1-4}
\multirow{3}{*}{COLING-2020}
    & Original       & 0.5075 & 0.4129 \\
    & +Critical Scoring & 0.4983 & 0.4049 \\
    & Change (\%)    & $-$1.81\% & $-$1.94\% \\
\bottomrule
\end{tabular}
\end{table}

\subsubsection{Sensitivity Analysis on Critical Scoring instruction}
\label{sssec:prompt_sensitivity}

Table \ref{tab:prompt_sensitivity} shows the performance of GPT-HSPIM\textsubscript{joint} with and without an additional instruction in the scoring prompt. 
Specially, we add the instruction ``Be decisive in your scoring and give higher or lower scores when there is clear evidence, rather than always choosing the middle range" to the scoring prompt to encourage critical evaluation. 
This experiment tests whether HSPIM is sensitive to the scoring instruction. The analysis also aims to find out if this instruction can reduce conservative scores and improve results.
We pay attention to the ACL-2017 dataset to see if this change helps HSPIM perform better.

The results show that adding the critical scoring instruction lowers RMSE and MAE on most datasets, suggesting HSPIM is sensitive to prompt wording. 
However, this approach does not improve results on ACL-2017. 
Such an approach can help address the narrow distribution of predicted scores shown in Fig. \ref{fig:distribution_plot}, but more effective strategies are still required. We plan to explore better solutions in future work.

\begin{table*}[ht!]
\centering
\caption{Effect of applying neutralization to abstract and introduction.}
\label{tab:neutralization_effect}
\scriptsize
% \setlength{\tabcolsep}{0.5pt}
% \resizebox{\textwidth}{!}{%
\begin{tabular}{llcccc}
\toprule
\multirow{2}{*}{Model} & \multirow{2}{*}{Dataset} 
& \multicolumn{2}{c}{HSPIM (All Sections)} 
& \multicolumn{2}{c}{HSPIM (\textit{Abstract+Introduction})} \\
\cmidrule(lr){3-4} \cmidrule(lr){5-6}
& & RMSE & Avg. Innov. Score & RMSE & Avg. Innov. Score \\
\midrule
\multirow{3}{*}{ICLR-2017}
& Original       & 0.8283 & 3.3399 & 0.8958 & 3.6914 \\
& \textit{+Neutralize}    & 0.8392 & 3.3094 & 0.9760 & 3.3235 \\
& \textit{Change (\%)}    & +1.31\% & -0.91\% & +8.95\% & -9.97\% \\
\midrule
\multirow{3}{*}{CoNLL-2016}
& Original       & 0.0638 & 3.3955 & 0.2716 & 3.6207 \\
& \textit{+Neutralize}    & 0.0661 & 3.3000 & 0.3402 & 3.0278 \\
& \textit{Change (\%)}    & +3.50\% & -2.81\% & +25.26\% & -16.38\% \\
\midrule
\multirow{3}{*}{ACL-2017}
& Original       & 1.0887 & 3.4777 & 1.0238 & 3.6513 \\
& \textit{+Neutralize}    & 1.1044 & 3.4545 & 1.1560 & 3.4675 \\
& \textit{Change (\%)}    & +1.45\% & -0.67\% & +12.91\% & -5.03\% \\
\midrule
\multirow{3}{*}{COLING-2020}
& Original       & 0.5202 & 3.2950 & 0.4990 & 3.4902 \\
& \textit{+Neutralize}    & 0.5224 & 3.2892 & 0.4988 & 3.4510 \\
& \textit{Change (\%)}    & +0.43\% & -0.18\% & +0.04\% & -1.12\% \\
\bottomrule
\end{tabular}
% }
\end{table*}

\subsubsection{Sensitivity Analysis on Self-Praising Expressions}
Scientific paper sections such as \textit{Abstract} and \textit{Introduction} often include self-praising expressions, such as ``significantly better" or ``a novel model." To analyze how such expressions affect the innovation scoring in HSPIM framework, we conduct a sensitivity analysis on DeepSeek-HSPIM\textsubscript{joint} and show the results in Table \ref{tab:neutralization_effect}. 
Note that this analysis only focus on testing the robustness of HSPIM to self-praising expressions. It does not assume that the replaced expressions are necessarily exaggerated or unreasonable.
Specifically, we use zero-shot LLMs to detect self-praising expressions in the \textit{Abstract} and \textit{Introduction} of each paper. Then, LLMs rewrite these expressions in a neutral tone. For example, LLMs may change ``significantly better” to ``better” and ``a novel model” to ``a new model.”

In Table \ref{tab:neutralization_effect}, \textit{+Neutralize} rewrites self-praising expressions in the \textit{Abstract} and \textit{Introduction} into more neutral statements using DeepSeek. 
We observe that \textit{+Neutralize} notably reduces the average innovation score for HSPIM with only \textit{Abstract + Introduction}. 
This suggests that zero-shot LLMs tend to assign higher scores when the text contains subjective or emphatic language. 
On the other hand, HSPIM with all section types (noted as HSPIM (All Sections)) is more robust, showing smaller changes in both innovation scores and RMSE.
HSPIM combines text chunks from all section types to improve the stability of the innovation scoring.

\begin{figure}[ht!]
  \centering
  \includegraphics[width=0.4\textwidth]{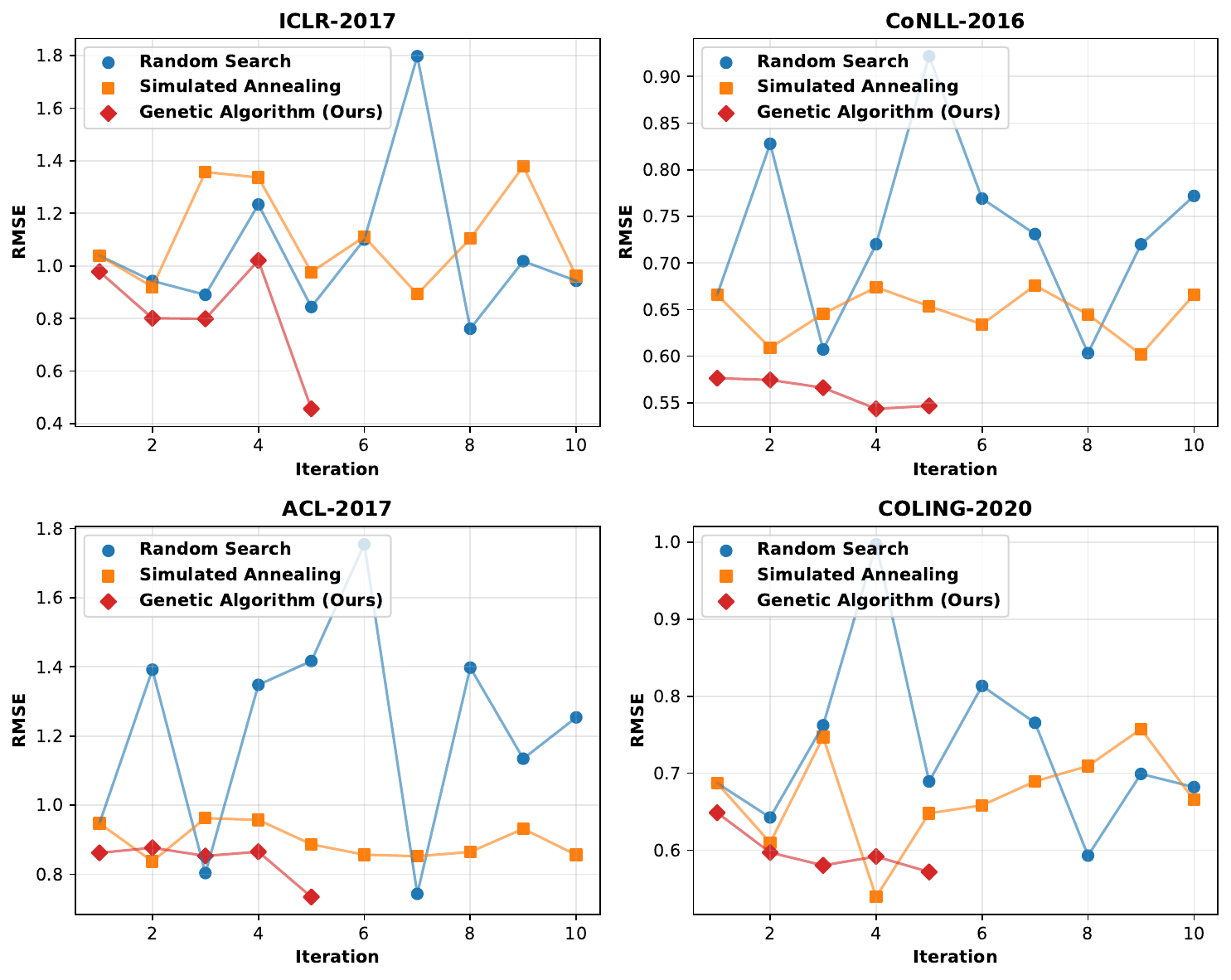}
  \caption{Training RMSE of Genetic Algorithm (ours), Random Search, and Simulated Annealing on HSPIM-DeepSeek. Each iteration randomly samples 20 training samples from the dataset.}
  \label{optimization_comparisons}
\end{figure}

\subsection{Comparison of Prompt Optimization Methods}
Fig. \ref{optimization_comparisons} compares our Genetic Algorithm (GA) with Random Search (RS) and Simulated Annealing (SA) \cite{kirkpatrick1983optimization} on HSPIM-DeepSeek. 

Notes Before Interpreting Results:
(i) To reduce training cost, we do not use the full training set in each iteration. Instead, we randomly divide the training set into batches of 20 samples and use a different batch for each iteration. All methods use the same batch within the same iteration.
So, RMSE may fluctuate across iterations, and it does not always decrease.
(ii) In HSPIM, GA uses 10 question-prompt combinations in each iteration (population size \(P = 10\)), and retains 2 elite individuals.
In contrast, RS and SA are limited by their algorithm design to test only one prompt combination per iteration and keep the best result found so far.
To ensure fairness, we set fewer iterations for GA (\(I = 5\)) and more for RS and SA (\(I = 10\)).  
(iii) RS and SA start from the same initial prompt combination. GA includes this prompt and adds 9 more random prompts. 
So, RS and SA give the same RMSE in iteration 1, while GA may show a different value.

In Fig. \ref{optimization_comparisons}, as the number of iterations increases, GA shows an overall decreasing trend in training RMSE. It achieves lower RMSE within fewer iterations. In comparison, RS shows large fluctuations without a clear trend. SA has smaller changes than RS, but it also does not show a clear improvement. The lack of consistent trends in RS and SA may be due to the limited number of iterations.
We also notice that since LLM prompting for different prompt combinations is independent, GA supports parallel evaluation of multiple prompts in each iteration.
Overall, GA proves to be a stable and efficient method for prompt optimization in HSPIM.

\begin{figure}[ht]
    \centering
    \subfloat[\scriptsize ICLR-2017]{
        \includegraphics[width=0.35\linewidth]{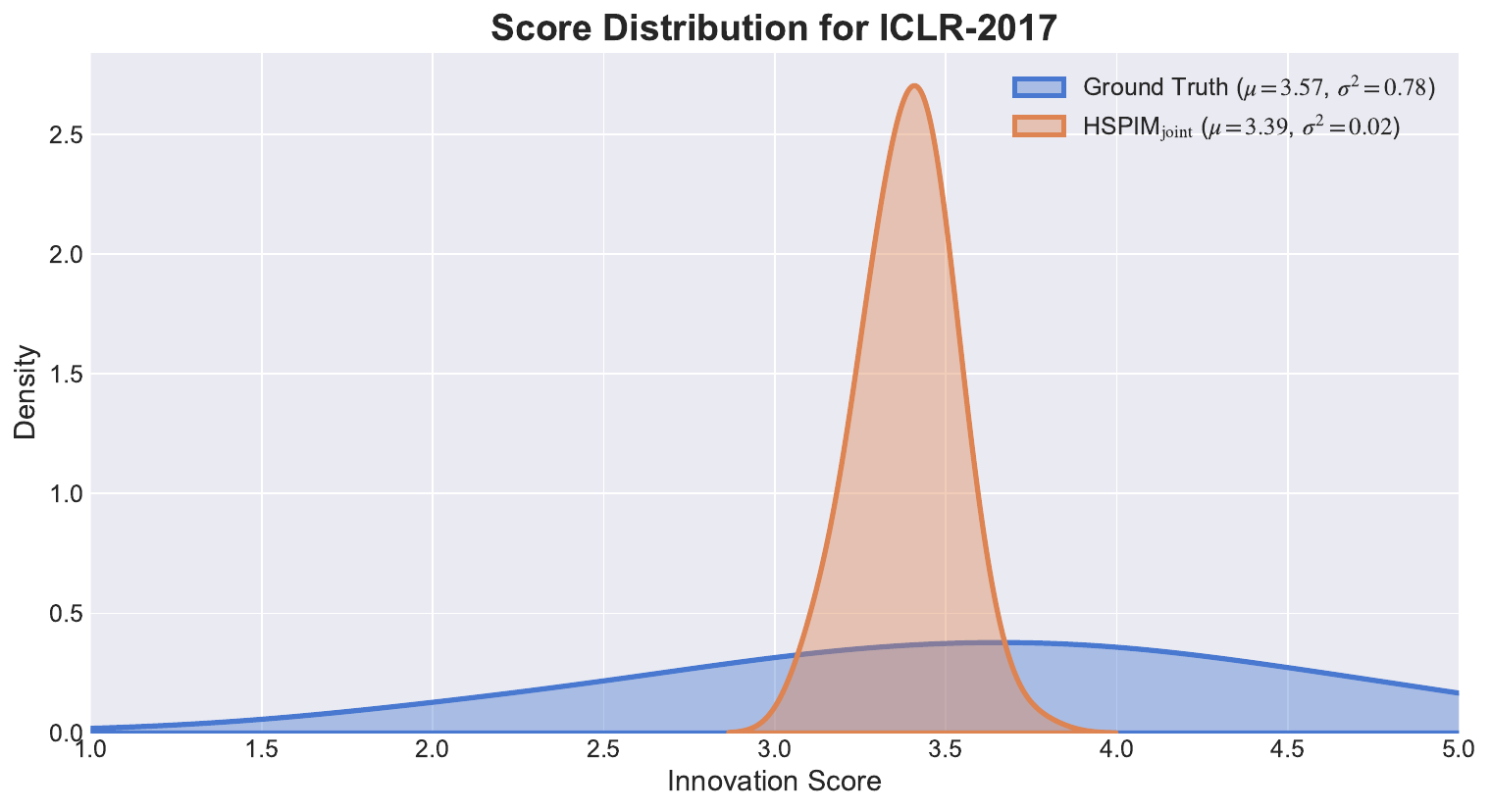}
        \label{fig:distribution_plot_iclr_2017}
    }
    \subfloat[\scriptsize CoNLL-2016]{
        \includegraphics[width=0.35\linewidth]{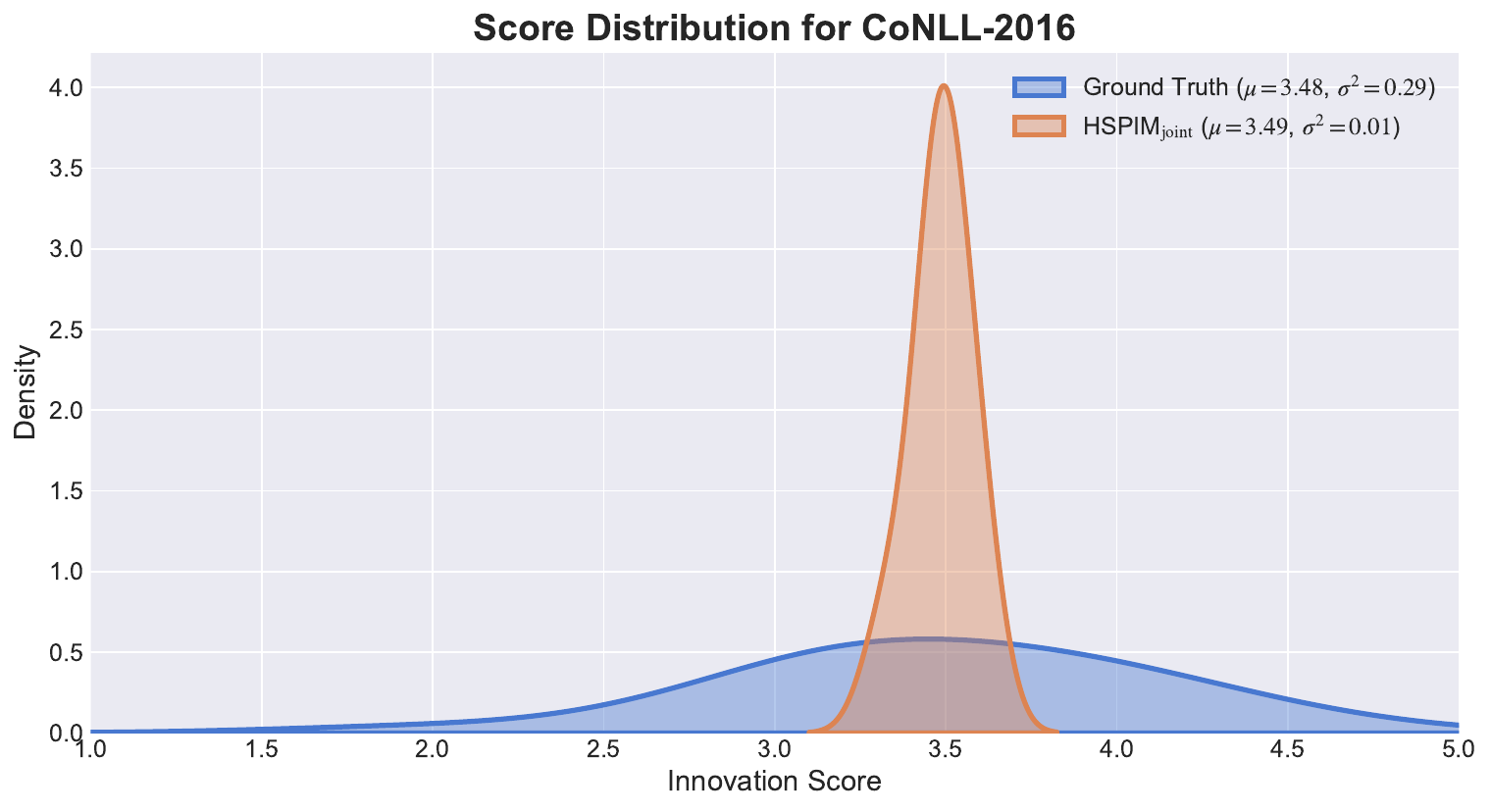}
        \label{fig:distribution_plot_conll_2016}
    }
    
    \subfloat[\scriptsize ACL-2017]{
        \includegraphics[width=0.35\linewidth]{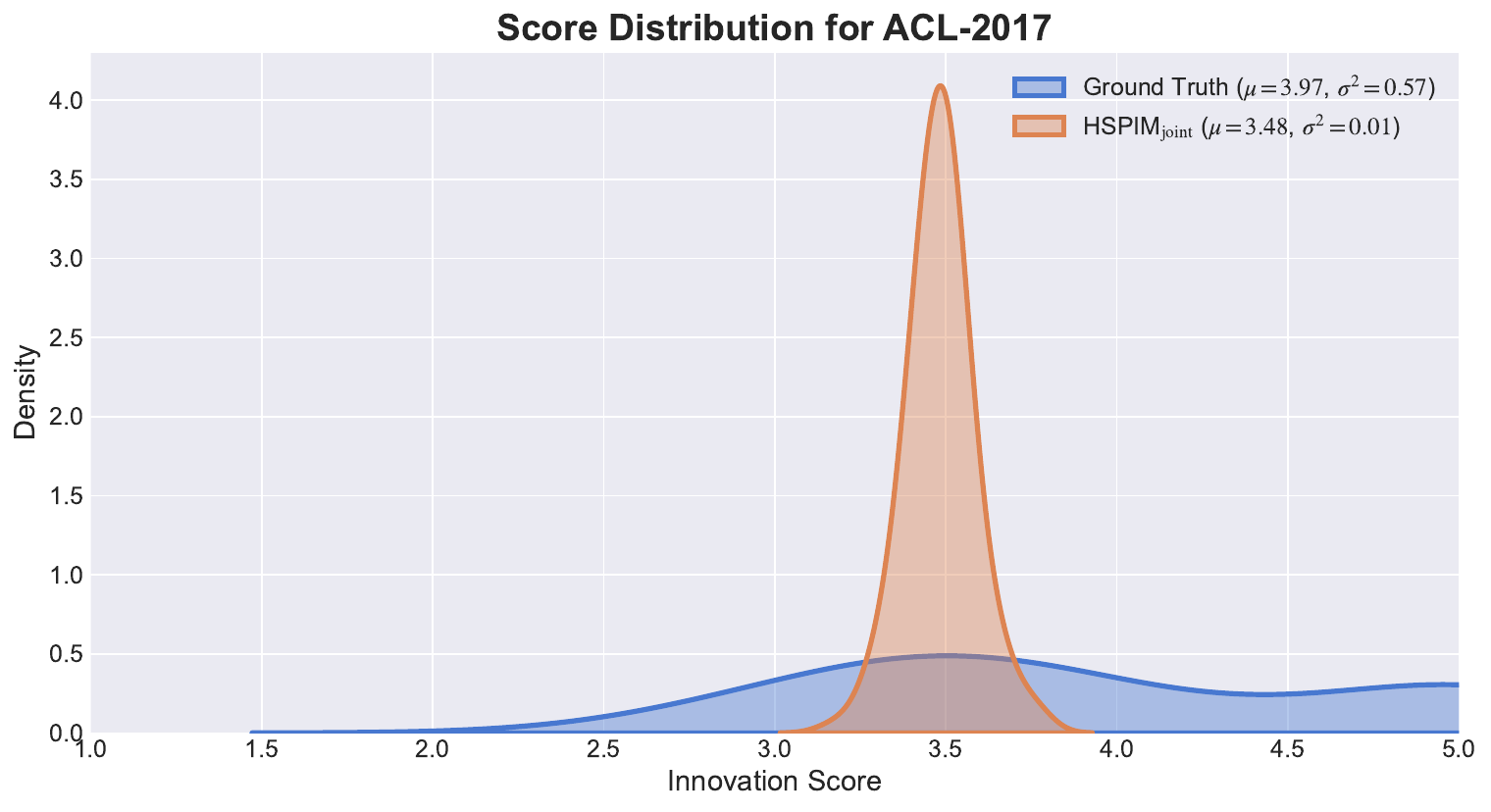}
        \label{fig:distribution_plot_acl_2017}
    }
    \subfloat[\scriptsize COLING-2020]{
        \includegraphics[width=0.35\linewidth]{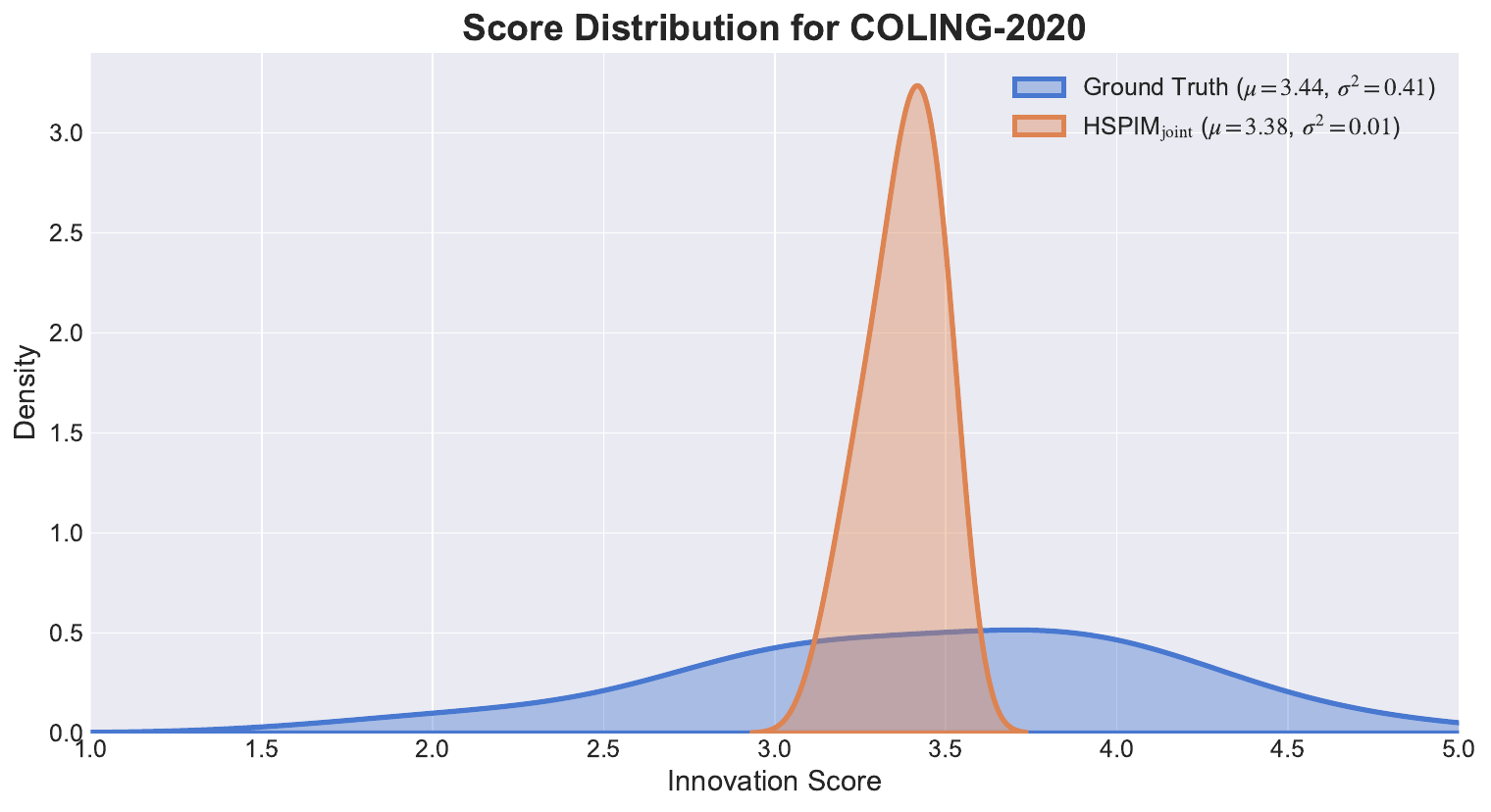}
        \label{fig:distribution_plot_coling_2020}
    }
    \caption{Comparison of ground-truth and HSPIM\textsubscript{joint} predicted innovation score distributions on four datasets. 
    Predicted scores are more concentrated around the mean.
    % showing reduced variance and fewer extreme values.
    }
    \label{fig:distribution_plot}
\end{figure}

\subsection{Analysis of Innovation Score Distributions}
\label{ssec:score_distribution}
Fig. \ref{fig:distribution_plot} compares the score distributions of ground-truth and HSPIM\textsubscript{joint} predictions on four datasets. 
Across all datasets, ground-truth scores exhibit a wider spread, including a noticeable proportion of both high and low (tail) values. 
In contrast, HSPIM\textsubscript{joint} predictions are more concentrated and have a lower mean and variance. 
For example, on ACL-2017, the predicted mean and variance are 3.48 and 0.01, while the ground-truth mean and variance are 3.97 and 0.57.
In HSPIM, LLMs rarely give extreme scores. The mismatch in distribution helps explain performance drops on datasets with more diverse or higher ground-truth scores.

\begin{table*}[ht!]
\centering
\caption{Misclassification handling and label-noise sensitivity for GPT-HSPIM\textsubscript{joint}. 
Noise is injected by randomly converting $p\%$ of section labels to \emph{unmatched}.
$\Delta_{\text{orig}}$: change vs \textit{Original}; $\Delta_{\text{handle}}$: change vs \textit{With misclassification handling}.}
\label{tab:misclass_effect}
\scriptsize
% \resizebox{\textwidth}{!}{
\begin{tabular}{l l cc}
\toprule
Dataset & Setting & RMSE & MAE \\
\midrule
\multirow{5}{*}{ICLR-2017}
    & Original & 0.8283 & 0.6966 \\
    & \textit{With misclassification handling} & 0.8248 {\scriptsize($\Delta_{\text{orig}}$ $-0.42\%$)} & 0.6951 {\scriptsize($\Delta_{\text{orig}}$ $-0.22\%$)} \\
    & \textit{+ Noise (unmatched 5\%)}  & 0.8376 {\scriptsize($\Delta_{\text{handle}}$ $+1.56\%$)} & 0.7064 {\scriptsize($\Delta_{\text{handle}}$ $+1.63\%$)} \\
    & \textit{+ Noise (unmatched 10\%)} & 0.8268 {\scriptsize($\Delta_{\text{handle}}$ $+0.25\%$)} & 0.6972 {\scriptsize($\Delta_{\text{handle}}$ $+0.30\%$)} \\
    & \textit{+ Noise (unmatched 20\%)} & 0.8430 {\scriptsize($\Delta_{\text{handle}}$ $+2.20\%$)} & 0.7079 {\scriptsize($\Delta_{\text{handle}}$ $+1.85\%$)} \\
\cmidrule(lr){1-4}
\multirow{5}{*}{CoNLL-2016}
    & Original & 0.0638 & 0.0560 \\
    & \textit{With misclassification handling} & 0.0798 {\scriptsize($\Delta_{\text{orig}}$ $+25.08\%$)} & 0.0659 {\scriptsize($\Delta_{\text{orig}}$ $+17.68\%$)} \\
    & \textit{+ Noise (unmatched 5\%)}  & 0.0798 {\scriptsize($\Delta_{\text{handle}}$ $+25.08\%$)} & 0.0659 {\scriptsize($\Delta_{\text{handle}}$ $+17.68\%$)} \\
    & \textit{+ Noise (unmatched 10\%)} & 0.0798 {\scriptsize($\Delta_{\text{handle}}$ $+25.08\%$)} & 0.0659 {\scriptsize($\Delta_{\text{handle}}$ $+17.68\%$)} \\
    & \textit{+ Noise (unmatched 20\%)} & 0.0798 {\scriptsize($\Delta_{\text{handle}}$ $+25.08\%$)} & 0.0659 {\scriptsize($\Delta_{\text{handle}}$ $+17.68\%$)} \\
\cmidrule(lr){1-4}
\multirow{5}{*}{ACL-2017}
    & Original & 1.0887 & 0.8572 \\
    & \textit{With misclassification handling} & 1.1491 {\scriptsize($\Delta_{\text{orig}}$ $+5.55\%$)} & 0.9284 {\scriptsize($\Delta_{\text{orig}}$ $+8.31\%$)} \\
    & \textit{+ Noise (unmatched 5\%)}  & 1.1526 {\scriptsize($\Delta_{\text{handle}}$ $+0.31\%$)} & 0.9270 {\scriptsize($\Delta_{\text{handle}}$ $-0.15\%$)} \\
    & \textit{+ Noise (unmatched 10\%)} & 1.1651 {\scriptsize($\Delta_{\text{handle}}$ $+1.39\%$)} & 0.9360 {\scriptsize($\Delta_{\text{handle}}$ $+0.82\%$)} \\
    & \textit{+ Noise (unmatched 20\%)} & 1.1566 {\scriptsize($\Delta_{\text{handle}}$ $+0.66\%$)} & 0.9248 {\scriptsize($\Delta_{\text{handle}}$ $-0.39\%$)} \\
\cmidrule(lr){1-4}
\multirow{5}{*}{COLING-2020}
    & Original & 0.5202 & 0.3933 \\
    & \textit{With misclassification handling} & 0.4806 {\scriptsize($\Delta_{\text{orig}}$ $-7.61\%$)} & 0.3874 {\scriptsize($\Delta_{\text{orig}}$ $-1.50\%$)} \\
    & \textit{+ Noise (unmatched 5\%)}  & 0.4758 {\scriptsize($\Delta_{\text{handle}}$ $-1.00\%$)} & 0.3846 {\scriptsize($\Delta_{\text{handle}}$ $-0.74\%$)} \\
    & \textit{+ Noise (unmatched 10\%)} & 0.4722 {\scriptsize($\Delta_{\text{handle}}$ $-1.75\%$)} & 0.3815 {\scriptsize($\Delta_{\text{handle}}$ $-1.53\%$)} \\
    & \textit{+ Noise (unmatched 20\%)} & 0.4810 {\scriptsize($\Delta_{\text{handle}}$ $+0.09\%$)} & 0.3835 {\scriptsize($\Delta_{\text{handle}}$ $-1.02\%$)} \\
\bottomrule
\end{tabular}
% }
\end{table*}

\subsection{Impact of Section Misclassification}
\label{ssec:misclassification}
To address the risk of inaccurate section assignment for ambiguous or unconventional section content, we introduce a revised classification strategy, referred to as \textit{With revised classification strategy} in Table \ref{tab:misclass_effect}.
Previously, the LLM was required to assign every section to the closest type in a fixed list, even when the fit was poor.
In this revised setting, HSPIM first attempts to match each section to the predefined types. 
If a section cannot be confidently classified, it is labeled as ``unmatched'' and, in the QA augmentation step of HSPIM, only uses the common question prompt without applying any section-specific questions.
This design allows HSPIM to more flexibly handle ambiguous or unconventional section content.
As shown in Table \ref{tab:misclass_effect}, this strategy maintains stable performance on most datasets.
In addition, we test robustness by randomly converting 5\%, 10\%, and 20\% of section labels to \emph{unmatched} and re-running QA and scoring. The results show that RMSE and MAE change only slightly, confirming resilience to occasional section misclassification.

\begin{table}[h]
\centering
\caption{Computational cost of different LLMs.}
\label{tab:computational_comparison}
\scriptsize
\begin{tabular}{lccc}
\toprule
LLM & Task & Avg. output Length & Avg Time (s) \\
\midrule
\multirow{2}{*}{Llama} & QA & 6011.61 & 35.29 \\
 & Scoring & 301.49 & 4.66 \\
\midrule
\multirow{2}{*}{Mistral} & QA & 2046.89 & 11.20 \\
 & Scoring & 606.31 & 4.78 \\
\midrule
\multirow{2}{*}{Qwen} & QA & 21626.53 & 106.52 \\
 & Scoring & 454.88 & 5.71 \\
\midrule
\multirow{2}{*}{Deepseek} & QA & 1737.63 & 20.80 \\
 & Scoring & 414.72 & 7.72 \\
\midrule
\multirow{2}{*}{GPT} & QA & 1911.99 & 6.36 \\
 & Scoring & 439.79 & 2.46 \\
\bottomrule
\end{tabular}
\end{table}

\subsection{LLM Computational Costs}
Table \ref{tab:computational_comparison} compares the computational costs of question answering and score generation for LLaMA-3-8B, Mistral-7B, Qwen2.5-7B, DeepSeek-V3, and GPT-4o mini. The metrics include average output length (in characters) and average response time per input. For DeepSeek and GPT, the reported time only reflects API response time. All experiments were conducted on an NVIDIA RTX 3090 GPU (24 GB) without any inference-time optimization. We observe that large-scale models tend to generate more output tokens.
Smaller LLMs (LLaMA-3-8B, Mistral-7B, Qwen2.5-7B) often suffer from repetition issues, leading to inflated output lengths.
Among the open-source small-scale models, Mistral-7B shows competitive time efficiency. Combined with its strong innovation scoring performance in Section~\ref{ssec:llm_hspim_comparison}, Mistral-7B stands out as a practical choice for HSPIM deployment.

\begin{figure}[ht]
  \centering
  \includegraphics[width=0.4\textwidth]{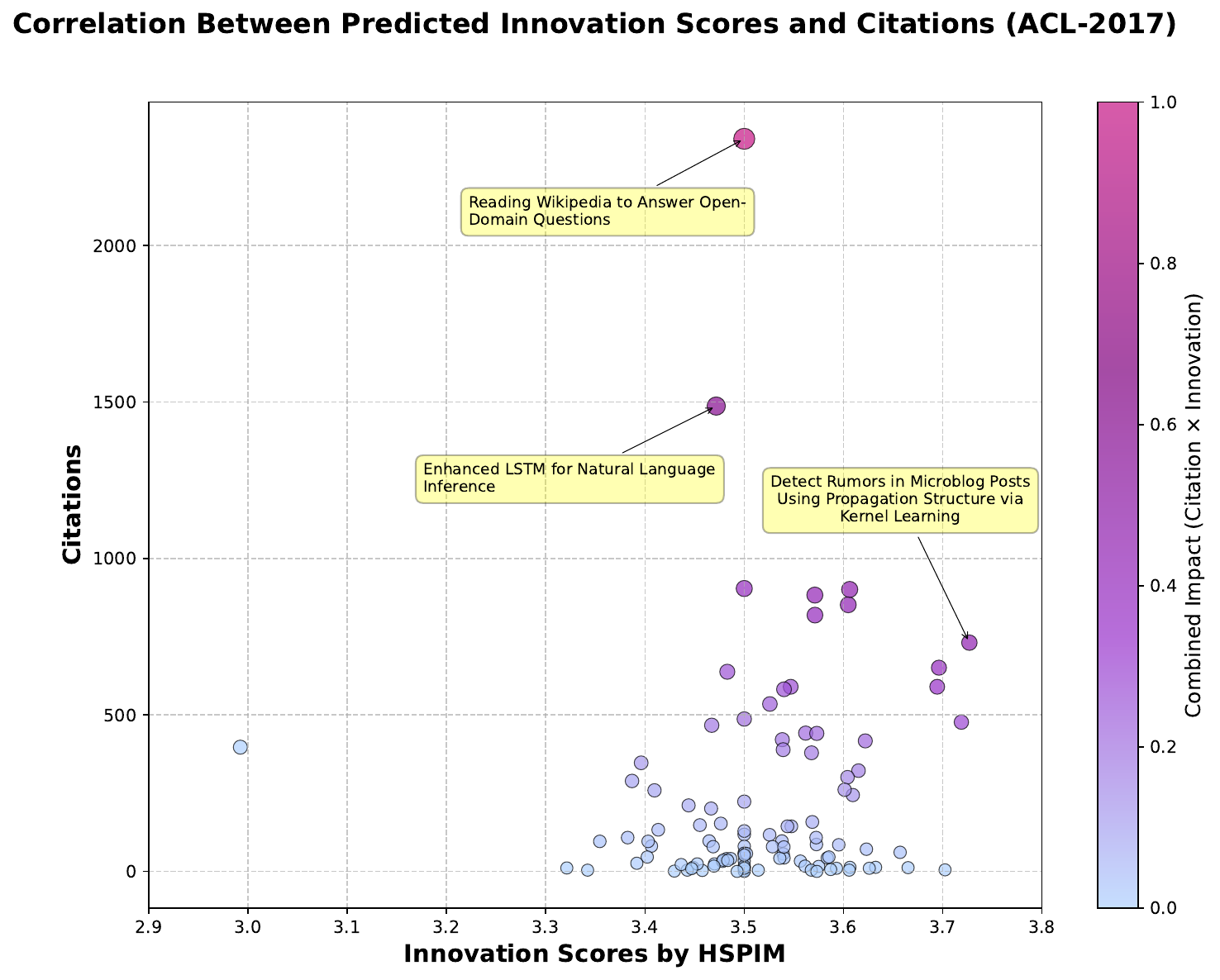}
  \caption{Correlation between HSPIM predicted innovation scores and citation counts on ACL-2017. Marker size is proportional to citations, color reflects the combined impact (citation $\times$ innovation), and annotations highlight the two most‑cited papers and the paper with the highest predicted innovation.}
  \label{innovation_vs_citations}
\end{figure}

\subsection{Case Study}
\subsubsection{HSPIM Innovation Scores vs. Citations}
In this section, we present a case study on the ACL-2017 dataset to compare the predicted innovation scores of HSPIM\textsubscript{joint} with citation counts.
We investigate the extent to which highly cited papers receive higher innovation scores.
Fig. \ref{innovation_vs_citations} presents the correlation between HSPIM innovation scores and citation counts for ACL-2017 papers. 
Citation data were collected from Google Scholar, with records updated as of March 2025.
We observe that papers with higher predicted innovation scores tend to appear slightly higher on the citation axis, suggesting a weak positive trend between predicted innovation and actual citation impact.

We observe that the two most-cited papers both receive a predicted innovation score of approximately 3.5. Their ground-truth innovation scores differ significantly: 3.5 for the most-cited paper and 5.0 for the second most-cited. 
This is because HSPIM uses a weighted average of section-based scores, which leads to stable outputs and avoids extreme values.
In fact, scientific paper quality, including innovation, does not always show a positive correlation with citation counts. External factors such as research domain or author reputation tend to influence citation numbers \cite{aksnes2019citations}.
The paper with the highest predicted innovation also shows a relatively high citation count. This suggests that HSPIM can partly identify papers with academic impact.

\begin{table*}[!htbp]
\centering
\caption{Best question-prompt combinations found by HSPIM\textsubscript{joint}.}
\label{tab:best_prompt_combos}
\renewcommand{\arraystretch}{0.6}
% \footnotesize
\scriptsize
\resizebox{\textwidth}{!}{
\begin{tabular}{p{2.8cm}|p{2.5cm}|p{11.0cm}} 
% \begin{tabular}
\toprule
\textbf{Dataset} & \textbf{Prompt Type} & \textbf{Question} \\
\midrule

% ====== ICLR-2017 ======
\multirow{2}{*}{\textbf{ICLR-2017}}
& \textbf{Common} 
& Additionally, briefly explain both: (1) Reasons for the lack of novelty; (2) Reasons for the presence of novelty. \\
\cmidrule{2-3}
& \textbf{Specific}
& \textbf{Abstract}: Do the contributions bring a major improvement over existing solutions?

\textbf{Introduction}: Assess the significance of the problem being addressed.

\textbf{Related Work}: Are there significant differences between this work and others in the field?

\textbf{Approach}: What new methods were used in this research? Do these methods have advantages over existing ones? Provide supporting cases.

\textbf{Analysis Theory}: Are the assumptions in the theoretical model different from those used in previous studies?

\textbf{Experiments}: How well do the results generalize to other settings?

\textbf{Experiment Analysis}: What new insights are revealed by the analysis of experimental results? Provide detailed comparisons with past methods using quantitative data.

\textbf{Discussion}: How do the limitations affect the applicability of the results?

\textbf{Conclusion}: What specific innovative findings are emphasized in the research conclusions? Is there clear innovation? \\
\midrule

% ====== CoNLL-2016 ======
\multirow{2}{*}{\textbf{CoNLL-2016}}
& \textbf{Common}
& Examine the research problem, providing reasons both for the lack of novelty and for any potential innovative contribution. \\
\cmidrule{2-3}
& \textbf{Specific}
& \textbf{Abstract}: Please summarize the main contributions and innovations proposed in the paper.

\textbf{Introduction}: Is the contribution clearly distinguishable from previous work?

\textbf{Related Work}: Critique the novelty of this research compared to related studies.

\textbf{Approach}: What new methods were used in this research? Do these methods have advantages over existing ones? Provide supporting cases.

\textbf{Analysis Theory}: How does the theory proposed here compare to existing theories in terms of scope and applicability?

\textbf{Experiments}: How well do the experimental results support the innovation of the research? Provide specific data and information on improvements.

\textbf{Experiment Analysis}: Evaluate whether the conclusions drawn from the analysis are well-supported.

\textbf{Discussion}: How do the limitations discussed affect the innovation of the research? Analyze these limitations and their consequences.

\textbf{Conclusion}: What specific innovative findings are emphasized in the research conclusions? Is there clear innovation? \\
\midrule

% ====== ACL-2017 ======
\multirow{2}{*}{\textbf{ACL-2017}}
& \textbf{Common}
& Provide arguments for both why this research may lack originality and why it could still be considered novel within its domain. \\
\cmidrule{2-3}
& \textbf{Specific}
& \textbf{Abstract}: Is there clear evidence to support the contributions as novel?

\textbf{Introduction}: Assess the significance of the problem being addressed.

\textbf{Related Work}: Are the comparisons to prior work sufficiently detailed?

\textbf{Approach}: What new methods were used in this research? Do these methods have advantages over existing ones? Provide supporting cases.

\textbf{Analysis Theory}: In what ways does the theory contribute to resolving any long-standing debates in the field?

\textbf{Experiments}: How well do the results generalize to other settings?

\textbf{Experiment Analysis}: What new insights are revealed by the analysis of experimental results? Provide detailed comparisons with past methods using quantitative data.

\textbf{Discussion}: Evaluate whether the limitations significantly hinder the novelty claims.

\textbf{Conclusion}: What specific innovative findings are emphasized in the research conclusions? Is there clear innovation? \\
\midrule

% ====== COLING-2020 ======
\multirow{2}{*}{\textbf{COLING-2020}}
& \textbf{Common}
& Provide arguments for both why this research may lack originality and why it could still be considered novel within its domain. \\
\cmidrule{2-3}
& \textbf{Specific}
& \textbf{Abstract}: Is there clear evidence to support the contributions as novel?

\textbf{Introduction}: Is the contribution clearly distinguishable from previous work?

\textbf{Related Work}: Are there significant differences between this work and others in the field?

\textbf{Approach}: What new methods were used in this research? Do these methods have advantages over existing ones? Provide supporting cases.

\textbf{Analysis Theory}: Is the methodology for the theoretical analysis clearly explained, and does it allow for reproducibility?

\textbf{Experiments}: How well do the results generalize to other settings?

\textbf{Experiment Analysis}: Evaluate whether the conclusions drawn from the analysis are well-supported.

\textbf{Discussion}: Evaluate whether the limitations significantly hinder the novelty claims.

\textbf{Conclusion}: Critically assess whether the findings represent a clear advancement in the field. \\
\bottomrule
\end{tabular}
}
\end{table*}

\subsubsection{Best Question-Prompt Combinations}
Table \ref{tab:best_prompt_combos} presents the best specific and common questions selected by DeepSeek-HSPIM\textsubscript{joint} for each dataset. 
In terms of specific questions, we observe that certain question patterns appear consistently, especially for \textit{Abstract} (``\textit{Is there clear evidence to support the contributions as novel?}") and \textit{Conclusion} (``\textit{What specific innovative findings are emphasized in the research conclusions?}").
These frequently selected question prompts suggest that innovation involves both identifying novelty and supporting the contributions with clear evidence.
The specific questions are designed to align with the content of each section, enabling LLMs to extract relevant signals of contribution and novelty. 

Additionally, the common question applies to all sections. It guides LLMs to assess both strengths and weaknesses of the paper. This helps reduce bias and ensures a balanced evaluation of innovation.
From Table \ref{tab:best_prompt_combos}, the common question ``\textit{Provide arguments for both why this research may lack originality and why it could still be considered novel within its domain}" appears twice. 
Overall, the genetic algorithm tends to select questions that ask authors to justify new ideas and reflect on their significance, which aligns with the goal of HSPIM to measure innovation with both novelty and contributions.

\section{Discussion}
In this section, we first explain how HSPIM addresses the limitations of existing content-based methods for measuring innovation. Then, we discuss potential directions for improving our HSPIM further.

\subsection{Addressing Challenges in Content-Based Innovation Measurement}

\begin{itemize}
    \item \textbf{How to handle long-text input?}
    HSPIM breaks down the full text into manageable section-based chunks and then generates QA pairs, forming a Paper-to-Sections-to-QAs decomposition.
    This hierarchical strategy simplifies long-text analysis by turning the global task into smaller sub-tasks.
    \item \textbf{How to capture the full scope of innovation?} Based on the innovation concept, HSPIM uses weighted innovation scoring to link section-level novelty to paper-level innovation. 
    It uses confidence scores as weights to measure each section's contribution to overall innovation.
    Table \ref{tab:ablation_confidence_score} shows the effectiveness of this weighted strategy.
    Section \ref{ssec:weighted_scoring_convergence} provides a theoretical explanation for the convergence of weighted scoring. In addition, QA augmentation helps to extract innovation-related information from text chunks for in-context learning.
    Table \ref{tab:HSPIM_framework_comparison} and Table \ref{tab:semantic_similarity} confirm its effectiveness in both quantifying innovation and improving interpretability.
    \item \textbf{How to improve model generalization?}
    HSPIM builds on zero-shot LLM prompting. Table \ref{tab:unfair_comparison} shows that HSPIM outperforms both supervised deep learning baselines and LLM-based baselines.
    We use a genetic algorithm to optimize question prompts for improving robustness.
\end{itemize}

\subsection{Limitations and Future Works}
\begin{itemize}
    \item \textbf{Assumption of Academic Integrity:} We assume that the scientific papers under evaluation are written with academic integrity. 
    That is, authors do not fabricate experimental results or hide existing related work. Without this assumption, papers that present fake results or ignore prior studies may receive inflated innovation scores.
    An automatic method to detect academic dishonesty could be a promising future direction.
    \item \textbf{Bias in LLM as a Judge:} 
    We observe that the zero-shot LLM in HSPIM tends to assign moderate scores and avoids giving very high scores, such as 5.0. This may be explained by the LLM's verbosity bias \cite{zheng2023judging}, which favors longer inputs. As a result, short section texts in HSPIM may receive conservative scores. 
    To address this, we plan to explore better ways to detect highly innovative papers in the future.
    \item \textbf{Bias in LLM as a Knowledge Base:} 
    HSPIM relies on the knowledge and reasoning abilities of pre-trained LLMs. Some papers in the dataset may already exist in the LLM's training data, which can introduce bias. To reduce this risk, we define clear scoring criteria and apply QA augmentation to encourage the LLM to focus on the presented content. Besides, for papers from newly emerging fields, existing LLMs may lack sufficient knowledge, making their evaluations unreliable. Future work can enhance LLMs with online retrieval-augmented generation to improve domain knowledge.
\end{itemize}

\subsection{Practicality and Deployment}

\begin{itemize}
    \item \textbf{Application potential:} HSPIM provides a paper-level innovation score together with QA pairs and scoring reasons that can be inspected alongside the paper. The framework is training-free, applicable across domains, runs end to end, and works with both API-based and open source LLMs (see Section \ref{sssec:step5_prompt_opt}).

    \item \textbf{Cost model:} For a paper \(p_i\) with \(l_i\) chunks, the inference time is \(\sum_{k=1}^{l_i} \tau_{\mathrm{QA},k} + \sum_{k=1}^{l_i} \tau_{\mathrm{score},k}\). Empirical times are reported in Table~\ref{tab:computational_comparison}. Reducing the number of chunks lowers cost roughly in proportion to the reduction. In our experiments, HSPIM\textsubscript{pruning} could improve accuracy by using fewer section types. Besides, chunks are independent and can be processed in parallel to reduce total runtime. Prompt optimization is done offline and inference uses the selected prompts, keeping the framework training-free.
\end{itemize}

\section{Conclusion}
In this paper, we propose HSPIM, a hierarchical framework to measure the innovation of scientific papers. 
HSPIM is the first approach to conduct a Paper-to-Sections-to-QAs decomposition for assessing scientific innovation. HSPIM classifies text chunks by section type and transforms the complex task of measuring paper-level innovation into simpler sub-tasks of section-level question answering and weighted innovation scoring. Unlike prior content-based models that evaluate novelty from specific sections rather than the full paper, or peer review LLMs that generate review texts and attribute scores, HSPIM directly targets innovation measurement with a training-free hierarchical design. HSPIM not only enables effective innovation quantification but also provides informative textual outputs, including section-based QA pairs and scoring rationales, which are missing in previous content-based regression models.
Extensive experiments demonstrate the effectiveness and interpretability of HSPIM in measuring innovation.

% \appendix
% \section{Example Appendix Section}


\bibliographystyle{elsarticle-num}
\bibliography{reference.bib}

\end{document}